\documentclass{article}
\usepackage{microtype}
\usepackage{fullpage}
\usepackage{graphicx}
\usepackage{subfigure}
\usepackage{booktabs}

\PassOptionsToPackage{noend}{algorithmic}

\usepackage[utf8]{inputenc}

\usepackage{amsmath,amssymb,amsthm}
\usepackage{natbib}
\usepackage{multicol}
\usepackage{multirow}
\usepackage{grffile}
\usepackage{float}
\usepackage{hyperref}
\usepackage{algorithm}
\usepackage[noend]{algorithmic}

\newcommand{\bbR}{\mathbb{R}}
\newcommand{\caD}{\mathcal{D}}
\newcommand{\eps}{\varepsilon}
\newcommand{\Galpha}{\max\left\{ \frac{\sqrt{G^2+(1-\alpha)^2}}{\alpha}, 1 \right\}}
\newcommand{\betaalpeps}{\frac{\beta+2G^2/\eps}{\alpha}}
\DeclareMathOperator*{\E}{\mathbf{E}}
\DeclareMathOperator{\VaR}{VaR}
\DeclareMathOperator{\CVaR}{CVaR}
\DeclareMathOperator{\CVaRtilde}{\widetilde{CVaR}}
\DeclareMathOperator{\regret}{regret}
\DeclareMathOperator{\proj}{proj}
\DeclareMathOperator{\diam}{diam}

\DeclareMathOperator{\argmin}{argmin}

\newtheorem{theorem}{Theorem}[section]
\newtheorem{lemma}[theorem]{Lemma}
\theoremstyle{definition}
\newtheorem{definition}[theorem]{Definition}

\usepackage{mathtools}
\DeclarePairedDelimiter{\norm}{\lVert}{\rVert}
\DeclarePairedDelimiter{\abs}{\lvert}{\rvert}
\DeclarePairedDelimiter{\inprod}{\langle}{\rangle}

\usepackage{xcolor}

\begin{document}
\title{Statistical Learning with Conditional Value at Risk}

\author{Tasuku Soma\\
The University of Tokyo\\
\texttt{tasuku\_soma@mist.i.u-tokyo.ac.jp}
\and
Yuichi Yoshida\thanks{Supported by JSPS KAKENHI Grant Number JP17H04676}\\
National Institute of Informatics\\
\texttt{yyoshida@nii.ac.jp}
}
\maketitle

\begin{abstract}
We propose a risk-averse statistical learning framework wherein the performance of a learning algorithm is evaluated by the conditional value-at-risk (CVaR) of losses rather than the expected loss.
We devise algorithms based on stochastic gradient descent for this framework.
While existing studies of CVaR optimization require direct access to the underlying distribution, our algorithms make a weaker assumption that only i.i.d.\ samples are given.
For convex and Lipschitz loss functions, we show that our algorithm has $O(1/\sqrt{n})$-convergence to the optimal CVaR, where $n$ is the number of samples.
For nonconvex and smooth loss functions, we show a  generalization bound on CVaR\@.
By conducting numerical experiments on various machine learning tasks, we demonstrate that our algorithms effectively minimize CVaR compared with other baseline algorithms.
\end{abstract}

\allowdisplaybreaks

\section{Introduction}

We consider decision making under a stochastic environment.
Let $\ell(\cdot\ ; z)\colon \bbR^d \to [0,1]$ be a loss function, where $z$ is a random variable distributed under some distribution $\caD$ describing the uncertainty of the environment or knowledge.
In standard statistical learning, the goal is to find $w$ in a set $K \subseteq \mathbb{R}^d$ that minimizes the expected loss $\E_{z \sim \caD}[\ell(w;z)]$ given i.i.d.\ samples from $\caD$.
The main challenge is to achieve \emph{generalization}, i.e., we want to guarantee that the empirical loss of the computed $w$ is close to the expected loss of $w$ with respect to $\mathcal{D}$.
In statistical learning theory, several algorithms have been shown to achieve generalization.

In many real-world decision-making tasks in finance, robotics, and medicine, we are often \emph{risk-averse}:
we want to minimize the probability of suffering from a considerable loss rather than simply minimizing the expected loss~\citep{Mansini2007,Yau2011,Tamar2015}.
In medical applications, for example, we must avoid catastrophic events such as fatality in patients.
This is also true in finance and robotics:
once we go bankrupt or destroy robots, we are no longer able to continue the process anymore.
Unfortunately, however, classical statistical learning theory does not control the risk of such rare but disastrous events.

\emph{Conditional value-at-risk (CVaR)} is a popular risk measure for such risk-averse applications~\cite{Rockafellar2000,Krokhmal2002}.
Formally, given a parameter $\alpha\in [0,1]$, the CVaR of $w$ is defined as the average of the worst $\alpha$-fraction of the losses, i.e.,
\[
    \CVaR_{\alpha,\caD}(w) = \E_{z\sim\caD}[\ell(w;z) \mid \ell(w;z)\geq \VaR_\alpha(w)],
\]
where $\VaR_\alpha(w)$ is the $(1-\alpha)$-quantile of the random variable $\ell(w;z)$, i.e.,
\[
    \VaR_{\alpha,\caD}(w) = \inf\left\{\tau \in \mathbb{R}: \Pr_{z\sim\caD}(\ell(w;z) \leq \tau) \geq 1-\alpha \right\} .
\]
Therefore, CVaR naturally captures the scenarios in which we incur a huge loss.
Note that, when $\alpha=1$, $\CVaR_{\alpha,\caD}$ coincides with the expected loss.
Thus, CVaR is a generalization of the expectation.
From the optimization perspective, CVaR admits beautiful connections to convex analysis, which makes CVaR easier to optimize compared to other risk measures~\citep{Rockafellar2000,Shapiro2014}.

However, most CVaR optimization literature usually assumes a stronger access model to the underlying distribution $\caD$ than i.i.d.\ samples.
For example, $\caD$ is explicitly provided as a discrete distribution or we have an oracle with which we can compute the expectation of a function under $\caD$.
This assumption is often unrealistic, and hence these existing studies cannot be directly applied to a statistical learning setting.
In particular, it is unclear \emph{whether CVaR generalizes}, i.e., if a solution computed from an empirical distribution achieves a small CVaR in the true distribution.

\subsection{Our Contribution}
In this paper, we study the optimization of CVaR from the viewpoint of statistical learning.
In the following, we fix $\alpha$ to be a constant and omit $\alpha$ from the notations.
\begin{itemize}
\itemsep=0pt
    \item When the loss function $\ell(\cdot\ ;z)$ is convex and $G$-Lipschitz for all $z$ and $K \subseteq \mathbb{R}^d$ is a convex set, we prove that a \emph{stochastic gradient descent (SGD)} algorithm finds $w \in K$ satisfying
    \[
        \E_{z_1,\dots,z_n}[\CVaR_\caD(w)] - \CVaR_\caD(w^*) \leq O\left(\frac{GD}{n^{1/2}}\right)
    \]
    given i.i.d.\ samples $z_1,\dots,z_n$,
    where $w^* \in \argmin_{w\in K}\CVaR_\caD(w)$ and $D$ is the diameter of the feasible set $K$.
    \item When the loss function $\ell(\cdot\ ;z)$ is a smooth nonconvex function for all $z$ and $K = \mathbb{R}^d$, we show that CVaR optimization can be reduced to optimize the expectation of an auxiliary loss function $f\colon (w, \tau) \mapsto [0,1]$. Then, assuming that the loss function $\ell(\cdot\ ; z)$ is $G$-Lipschitz and $\beta$-smooth, we show that SGD finds $(w,\tau)$ such that
    \[
        \E_{z}[f(w,\tau; z)] - \CVaR_\caD(w) \leq O\left(\frac{G\beta^{1/2}}{n^{1/4}}+\frac{G^{4/3}}{n^{1/6}}\right).
    \]
    Given the practical performance of SGD in nonconvex optimization, these results show that the CVaR attained by SGD generalizes even for nonconvex loss functions.
    \item In numerical experiments using real-world tasks, we demonstrate that our algorithm minimizes CVaR effectively compared to other baseline methods. Interestingly, in some classification tasks, our algorithms achieve \emph{better accuracy} than other standard algorithms that optimize the expected loss.
    We believe that learning with CVaR sheds new light on the statistical learning theory for classification tasks.
\end{itemize}

\subsection{Related Work}
CVaR was introduced by \citet{Rockafellar2000} as an example of \emph{coherent risk measures}~\cite{Artzner1999} in portfolio optimization.
Since then, CVaR has been successfully applied to machine learning.
\citet{Gotoh2016} studied SVM algorithms from the viewpoint of CVaR optimization.
\citet{Chow2014,Chow2015} employed CVaR optimization in reinforcement learning.
Several authors studied CVaR optimization in influence maximization and more broadly, submodular maximization~\cite{Maehara2015,Ohsaka2017,Wilder2018}.

As mentioned above, literature on CVaR optimization in statistical learning remains limited.
\citet{Tamar2015} studied CVaR optimization over i.i.d.\ samples and analyzed stochastic gradient descent under the assumption that $\CVaR_\caD$ is continuously differentiable, which is not true in general even if $\ell(\cdot\ ; z)$ is so.
The most relevant work to ours is a very recent paper by~\citet{Cardoso19a}.
They defined a concept called the CVaR regret for convex loss functions and provided online algorithms for minimizing the CVaR regret under bandit feedback.
To deal with limited feedback, their algorithms are quite different from our SGD algorithms.
We provide a sharper bound than their methods, although the learner has more information in our setting.
Further, we consider nonconvex loss functions while they focus on convex loss functions.

We note that our framework is completely different from quantile regression~\cite{Koenker2001}: the goal of the former is to minimize the CVaR of losses whereas that of the latter is to estimate the conditional quantile of the response variable across values of the predictor variables.

\subsection{Organization}
The remainder of this paper is organized as follows:
We introduce notions used throughout this paper in Section~\ref{sec:pre};
learning algorithms for the case that each loss function is convex are provided in Section~\ref{sec:convex};
the nonconvex case is discussed in Section~\ref{sec:nonconvex};
and our experimental results are presented in Section~\ref{sec:experiments}.
Finally, we conclude our paper in Section~\ref{sec:conclusions}.

\section{Preliminaries}\label{sec:pre}
For a positive integer $n$, let $[n] = \{1,2,\ldots,n\}$.
We denote the Euclidean norm by $\norm{\cdot}$.
A function $f\colon \bbR^d \to \bbR$ is said to be \emph{$G$-Lipschits} if $\abs{f(x) - f(y)} \leq G \norm{x - y}$ for all $x, y \in \mathbb{R}^d$.
When $f$ is convex, it is equivalent to $\norm{g} \leq G$ for any subgradient $g$ of $f$.
A function $f\colon \mathbb{R}^d \to \mathbb{R}$ is said to be \emph{$\beta$-smooth} if it is continuously differentiable and $\norm{\nabla f(x) - \nabla f(y)} \leq \beta \norm{x - y}$ for all $x, y \in \mathbb{R}^d$.
The projection of $x$ on a convex set $K$ is denoted by $\proj_K(x)$.
The diameter of a set $K \subseteq \bbR^d$ in the Euclidean distance is denoted by $\diam(K)$.

Let $\caD$ be a distribution, $z$ be a random variable distributed under $\caD$, $K \subseteq \bbR^d$ be a convex set, $\ell(\cdot\ ; z)\colon \bbR^d \to [0,1]$ be a loss function parameterized by $z$, and $\alpha \in [0,1]$.
For $x \in \mathbb{R}$, define ${[x]}_+ = \max\{x,0\}$.
Then, we can characterize $\CVaR_\mathcal{D}(w)$ as follows.
\begin{lemma}[{\citet{Rockafellar2000}}]\label{lem:Rockafellar2000}
We have
\[
    \CVaR_\mathcal{D}(w) = \min_{\tau \in \mathbb{R}} \frac{1}{\alpha}\E_{z \sim \caD}{[\ell(w; z) - \tau]}_+ + \tau.
\]
Furthermore, if $\ell(\cdot\ ;z)$ is convex for all $z$, then $\CVaR_\caD(w)$ is convex again in $w$.
\end{lemma}
Given samples $S = \{z_1, \dots, z_n\}$ of $\caD$, the \emph{empirical CVaR} is defined as
\[
    \CVaR_S(w) = \min_{\tau} \frac{1}{\alpha n}\sum_{i=1}^n {[\ell(w; z_i) - \tau]}_+ + \tau.
\]
Let $f: \bbR^{d}\times [0,1] \to [0,1]$ be
\[
    f(w,\tau; z) = \frac{1}{\alpha}{[\ell(w; z_i) - \tau]}_+ + \tau,
    \]
where $z$ is a parameter.
It is often convenient to work with $f$ rather than $\ell$.
The following is standard, and a proof can be found in the appendix.

\begin{lemma}\label{lem:Lips_const}
    Suppose that $\ell(\cdot\ ; z)$ is $G$-Lipschitz for all $z$. Then, $f(\cdot\ ; z)$ is $G_\alpha$-Lipschitz for all $z$, where $G_\alpha = \Galpha$.
\end{lemma}

\subsection{Online Convex Optimization}
For our analysis on convex loss functions, we use the framework of \emph{online convex optimization (OCO)}, which considers the following repeated game between a player and an adversary.
The player is given a convex set $K \subseteq \bbR^d$ in advance.
For each round $t=1,\dots,T$, the player plays $x_t \in K$ and the adversary selects a convex function $f_t\colon K \to [0,1]$.
Then, the player suffers from the loss $f_t(x_t)$ and the function $f_t$ is revealed to the player.
The goal of the player is to minimize the \emph{regret}:
\[
    \regret(f_1, \dots, f_T) = \sum_{t=1}^T f_t(x_t) - \min_{x^* \in K}\sum_{t=1}^T f_t(x^*).
\]
Note that $f_t$ can depend on the previous choices $x_1, \dots, x_{t-1}$ of the player.
In particular, if the algorithm of the player is deterministic, $f_t$ can also depend on $x_t$ because the adversary can infer the next play of the player.
For further details of OCO, we refer to the monograph of \citet{Hazan2016OCO}.


\section{Convex Loss}\label{sec:convex}
In this section, we show that learning with OCO generalizes with respect to CVaR, assuming that the loss function $\ell(\cdot\ ;z)$ is convex for every $z$.
We consider two classes of algorithms: online (Section~\ref{sec:online}) and offline (Section~\ref{subsec:offline}).
In the online setting, samples can arrive sequentially and the algorithms do not maintain the past samples.
Online algorithms are widely used to process large data owing to their memory efficiency.
In the offline setting, all samples $z_1,\dots,z_n$ are given as input.
Offline algorithms can process each sample multiple times to achieve better performance, which online algorithms are incapable of.

\subsection{Online Algorithms}\label{sec:online}
\subsubsection{General Framework}
We describe the general framework that our algorithms are based on.
For the $i$-th sample $z_i$, we define a function $f_i\colon K \times [0,1] \to \mathbb{R}$ as
\[
  f_i(w,\tau) = \frac{1}{\alpha} {[\ell(w; z_i) - \tau]}_+ + \tau.
\]
Next, apply an OCO algorithm $A$ on sequence $f_1, \dots, f_n$ to produce sequence $x_1 = (w_1,\tau_1), \dots, x_n = (w_n,\tau_n)$.
Finally, we output a vector $A(z_1,\ldots,z_n) := \frac{1}{n}\sum_{t=1}^n w_i$.
The regret of $A$ is now equal to
\[
    \regret_A(z_1,\ldots,z_n) = \sum_{i=1}^n f_i(x_i) - \min_{x^* \in K \times [0,1] }\sum_{i=1}^n f_i(x^*),
\]
where we slightly change the notation from $\regret_A(f_1,\ldots,f_n)$ to $\regret_A(z_1,\ldots,z_n)$  to emphasize the dependence on $z_1,\dots,z_n$.
This is convenient when we take the expectation over $z_1,\dots,z_n$.

We can bound the generalization error of the learned parameter $w \in K$ with respect to CVaR using regrets.
\begin{theorem}\label{thm:cvar-generalization-via-regret}
  For an OCO algorithm $A$ and a positive integer $n$, we have
  \begin{align*}
    & \E_{z_1,\dots,z_n \sim \mathcal{D}}[ \CVaR_\mathcal{D}(A(z_1,\ldots,z_n)) ] - \min_{w^* \in K}\CVaR_\mathcal{D}(w^*)
    \leq \frac{\E_{z_1,\dots,z_n \sim \mathcal{D}}[\regret_A(z_1,\ldots,z_n)]}{n}.
  \end{align*}
\end{theorem}
\begin{proof}
  We show the claim using a technique called \emph{online-to-batch}~\citep{Cesa2002}.
  Let us fix $w^* \in \argmin_{w \in K} \CVaR_\caD(w)$ and $\tau^*$ to be the optimal threshold corresponding to $w^*$, i.e., $\tau^*$ is chosen such that
  \[
    \CVaR_{\mathcal{D}}(w^*) = \frac{1}{\alpha} \E_{z\sim\caD}{[\ell(w^*; z) - \tau^*]}_+ + \tau^*,
  \]
  and let $x^* = (w^*, \tau^*)$.
  By the definition of regret, for any $z_1,\ldots,z_n$, we have
  \[
    \frac{1}{n}\sum_{i=1}^n f_i(x_i) - \frac{1}{n} \sum_{i=1}^n f_i(x^*) \leq \frac{\regret_A(z_1,\ldots,z_n)}{n}.
  \]
  Taking the expectation over $z_1, \dots, z_n$, we obtain the following:
  \begin{align*}
    & \frac{1}{n}\E_{z_1,\dots,z_n}\left[\sum_{i=1}^n f_i(x_i) \right] - \frac{1}{n} \E_{z_1,\dots,z_n}\left[\sum_{i=1}^n f_i(x^*) \right]
    \leq \frac{\E_{z_1,\dots,z_n}[\regret_A(z_1,\ldots,z_n)]}{n}.
  \end{align*}
  Now, we bound the two terms on the left-hand side;
  first, for each $i$, we have
  \begin{align*}
    &\E_{z_1,\dots,z_i}\left[f_i(x_i) \right]
    = \E_{z_1,\dots,z_{i-1}}\left[ \E_{z_i}[f_i(x_i) \mid z_1, \dots, z_{i-1} ] \right] \\
    &= \E_{z_1,\dots,z_{i-1}}\left[ \E_{z}\left[ \frac{1}{\alpha}{[\ell(w_i,z) - \tau_i]}_+ + \tau_i \right] \right] \tag{since $ (w_i, \tau_i) $ is independent from $z_i$} \\
    &\geq \E_{z_1,\dots,z_{i-1}}\left[ \min_{\tau}\E_{z}\left[ \frac{1}{\alpha}{[\ell(w_i,z) - \tau]}_+ + \tau \right] \right] \\
    &= \E_{z_1,\dots,z_{i-1}}[ \CVaR_\mathcal{D}(w_i) ].
  \end{align*}
  Thus,
  \begin{align*}
    & \E_{z_1,\dots,z_n}\left[ \frac{1}{n}\sum_{i=1}^n f_i(x_i) \right]
    \geq \E_{z_1,\dots,z_n}\left[ \frac{1}{n}\sum_{i=1}^n \CVaR_\mathcal{D}(w_i) \right]
    \geq \E_{z_1,\dots,z_n}\left[ \CVaR_\mathcal{D}(w) \right],
  \end{align*}
  where the second inequality follows from Jensen's inequality.

  Next, for each $i$, we have
  \begin{align*}
    & \E_{z_i}[f_i(x^*)]
    = \frac{1}{\alpha} \E_{z_i}{[\ell(w^*; z_i) - \tau^*]}_+ + \tau^*  \\
    &= \frac{1}{\alpha} \E_{z} {[\ell(w^*; z) - \tau^*]}_+ + \tau^*  \tag{since $ (w^*, \tau^*) $ is a constant.} \\
    &= \CVaR_\mathcal{D}(w^*). \tag{by the definition of $\tau^*$}
  \end{align*}
  Thus,
  \begin{align*}
    \frac{1}{n}\E_{z_1,\dots,z_n}\left[\sum_{i=1}^n f_i(x^*) \right]
    = \frac{1}{n}\sum_{i=1}^n \E_{z_i}[f_i(x^*)]
    = \CVaR_\mathcal{D}(w^*).
  \end{align*}
  This completes the proof.
\end{proof}

\subsubsection{Online Gradient Descent}
We use online gradient descent~\citep{Zinkevich2003} as the OCO algorithm for the instantiation of the general framework.
\begin{algorithm}[h!]
  \caption{Online Gradient Descent for CVaR}\label{alg:SGD}
  \begin{algorithmic}[1]
    \REQUIRE The number of iterations $n$.
    \STATE Initialize $x_1 = (w_1,\tau_1) \in K \times [0,1]$ arbitrarily.
    \STATE $\eta \leftarrow \frac{\sqrt{D^2+1}}{G_\alpha\sqrt{n}}$.
    \FOR{$i=1,\dots,n$}
      \STATE Observe a sample $z_i$.
      \STATE Compute $g_i \in \partial_x f(x_i; z_i)$ and update $x_{i+1} = \proj_K(x_i - \eta g_i)$.
    \ENDFOR
    \RETURN $w = \frac{1}{n}\sum_{i=1}^n w_i$
  \end{algorithmic}
\end{algorithm}

\begin{theorem}\label{thm:convex-SGD}
Assume that the loss function $\ell(\cdot\ ; z)$ is convex, $G$-Lipschitz, and $\ell(w;z) \in [0,1]$ for all $w$ and $z$, and the feasible region $K \subseteq \bbR^d$ is a convex set such that $\norm{w-w'} \leq D$ for all $w,w' \in K$.
Then, Algorithm~\ref{alg:SGD} outputs $w \in K$ such that
\begin{align*}
    &\E_{z_1,\dots,z_n} [ \CVaR_\mathcal{D}(w) ] - \min_{w^* \in K} \CVaR_\mathcal{D}(w^*) \\
    &\quad\leq \frac{G_\alpha\sqrt{D^2+1}}{\sqrt{n}},
\end{align*}
where $G_\alpha = \Galpha$.
\end{theorem}
\begin{proof}
  By the regret guarantee of online gradient descent~\citep{Zinkevich2003}, we have
  \[
    \regret(z_1,\ldots,z_n) \leq \sum_{i=1}^n \frac{\eta}{2} \norm{g_i}^2 + \frac{{\diam(K \times [0,1])}^2}{2\eta}.
  \]
  Thus,
  \begin{align*}
    & \E_{z_1,\dots,z_n}[\regret(z_1,\ldots,z_n)]
    \leq \sum_{i=1}^n \frac{\eta}{2} \E_{z_1,\dots,z_n}[\norm{g_i}^2] + \frac{D^2 + 1}{2\eta}
    \leq G_\alpha \frac{\eta n}{2} + \frac{D^2 + 1}{2\eta},
  \end{align*}
  where, in the second inequality, we used $\E[\norm{g_i}^2] \leq G_\alpha$ by Lemma~\ref{lem:Lips_const}.
  Now the claim is immediate by Theorem~\ref{thm:cvar-generalization-via-regret} and the choice $\eta = \frac{\sqrt{D^2+1}}{G_\alpha\sqrt{n}}$.
\end{proof}

We remark that one can use online mirror descent~\citep{Hazan2016OCO} instead of online gradient descent to obtain a similar bound that yields a better parameter dependence for certain settings.
We omit the details here because it immediately follows from our general framework.

\subsection{Offline Algorithms}\label{subsec:offline}
In this section, we discuss the offline algorithms for minimizing CVaR.
We consider stochastic gradient descent (SGD):
In this method, we update $x_{t+1} = \proj_K(x_t - \eta g_t)$ for $t = 1, \dots, T$, where $g_t \in \partial_x f(x_t; z_t)$ and $z_t$ is a uniform random variable over given $n$ samples $z_1, \dots, z_n$.
Unfortunately, the above-mentioned framework based on OCO breaks down when $T > n$, i.e., when we use each sample more than once.
However, we can still prove a similar (but slightly worse) bound even in this setting with an additional smoothness assumption.

\subsubsection{Smooth Approximation to CVaR}
The technical difficulty in the offline setting is that the auxiliary function $f(x;z)$ is nonsmooth even if $\ell$ is smooth.
This prevents us from using a generalization bound of SGD~\citep{Hardt2016}.
We address this issue using a \emph{smoothed plus function}.

\begin{lemma}[Folklore]\label{lem:folklore}
  For any $\eps > 0$, there exists a $2/\eps$-smooth convex function $\rho_\eps(s)\colon \bbR\to\bbR_+$ such that ${[s]}_+ \leq \rho_\eps(s) \leq {[s]}_+ + \eps$ for any $s \in \bbR$.
\end{lemma}

The examples of smoothed plus functions are the \emph{soft ReLu function}
\begin{align*}
  \rho_\eps(s) = \eps\log(1+e^{s/\eps})
\end{align*}
and the \emph{piecewise quadratic smoothed plus function}~\citep{Alexander2006}
\begin{align*}
  \rho_\eps(s) =
  \begin{cases}
  s, & (s \geq \eps) \\
  \frac{s^2}{4\eps} + \frac{s}{2} + \frac{\eps}{4}, & (-\eps \leq s \leq \eps) \\
  0. & (s \leq -\eps)
  \end{cases}
\end{align*}
Note that for both choices, $\rho_\eps$ is $2/\eps$-smooth.

We fix $\rho_\eps$ to be a function satisfying the condition in Lemma~\ref{lem:folklore}.
Now, we define a \emph{smoothed auxiliary function} $\tilde{f}_\eps(\cdot\ ; z)\colon \mathbb{R}^d \times [0,1] \to \mathbb{R}$ as
\[
  \tilde{f_\eps}(x=(w,\tau);z) = \frac{1}{\alpha}\rho_\eps(\ell(w;z) - \tau) + \tau.
\]

\begin{lemma}\label{lem:smoothing}
  If $\ell(\cdot\ ;z)$ is $G$-Lipschitz and $\beta$-smooth, then
  \begin{enumerate}
    \item $f(x;z) \leq \tilde{f_\eps}(x;z) \leq  f(x;z) + \eps$ for all $x$ and $z$,
    \item $f(x;z)$ is $G_\alpha$-Lipschitz.
    \item $\tilde{f_\eps}(x;z)$ is $\betaalpeps$-smooth for all $z$.
  \end{enumerate}
\end{lemma}

The \emph{smoothed CVaR} is defined as
\[
  \CVaRtilde_{\caD, \eps}(w) = \min_{\tau} \E_{z \sim \caD} \tilde{f_\eps}(w,\tau; z).
\]

\begin{lemma}\label{lem:smoothing-gap-CvaR}
  For any $w \in \mathbb{R}^d$, $\CVaR_\caD(w) \leq \CVaRtilde_{\caD, \eps}(w) \leq \CVaR_\caD(w) + \eps$.
\end{lemma}

\subsubsection{Stochastic Gradient Descent}
Now, we describe our SGD algorithm in Algorithm~\ref{alg:batch-SGD} and provide its generalization bound.

\begin{algorithm}[h!]
  \caption{Smoothed Stochastic Gradient Descent for CVaR}\label{alg:batch-SGD}
  \begin{algorithmic}[1]
    \REQUIRE Samples $S = \{z_1,\dots,z_n\}$ and the number of iterates $T$
    \STATE Initialize $x_1 = (w_1,\tau_1) \in K \times [0,1]$ arbitrarily.
    \STATE Set $\eta \gets \frac{\sqrt{D^2+1} \sqrt{n}}{G_\alpha \sqrt{T(n+2T)}}$ and $\eps \gets 2G_\alpha^2 \eta$.
    \FOR{$t=1,\dots,T$}
      \STATE Sample $z_t \sim S$ and take $g_t \in \partial_x \tilde{f}_\eps(x_t;z_t)$
      \STATE Update $x_{t+1} = \proj_K(x_t - \eta g_t)$
    \ENDFOR
    \RETURN $w = \frac{1}{T}\sum_{t=1}^{T} w_t$
  \end{algorithmic}
\end{algorithm}

\begin{theorem}\label{thm:convex-smooth-SGD}
Assume that the same assumption as in Theorem~\ref{thm:convex-SGD} holds and  $\ell(\cdot\ ;z)$ is $\beta$-smooth for all $z$.
Suppose that we run Algorithm~\ref{alg:batch-SGD} over a set $S$ of $n$ samples with $T=cn$ for $c > 0$, and let $w$ be the average of $T$ iterations of SGD\@.
If $c$ is sufficiently large such that
\begin{align}\label{eq:cond-on-c}
\frac{\sqrt{D^2+1}}{G_\alpha \sqrt{c(1+2c)n}} \leq \frac{\alpha}{\beta},
\end{align}
then
\begin{align*}
  & \E_{S,w}[\CVaR_\caD(w) - \min_{w^*\in K}\CVaR_S(w^*)]
  \leq \frac{G_\alpha \sqrt{D^2+1}}{\sqrt{n}} \left(\sqrt{\frac{1+2c}{c}} + \frac{2}{\sqrt{c(1+2c)}} \right).
\end{align*}
where the expectation is taken over $S$ and the randomness in the algorithm.
\end{theorem}
Note that
\[
\sqrt{\frac{1+2c}{c}} + \frac{2}{\sqrt{c(1+2c)}} \leq \frac{5\sqrt{3}}{3} \quad (\forall c \geq 1),
\]
and thus this bound is worse by only a constant factor compared to the online setting, that is, $T = n$ (Theorem~\ref{thm:convex-SGD}).
On the other hand, our offline bound holds even if $T > n$.

Next, we sketch the proof of this theorem.
The omitted proofs can be found in the appendix.

Let $R(\cdot) = \E_{z\sim\caD}[f(\cdot\ ;z)]$ and $R_S(\cdot) = \frac{1}{n}\sum_{i=1}^n f(\cdot\ ;z_i)$.
Further, we define smoothed versions $\tilde{R}_\eps(\cdot) = \E_{z\sim\caD}[\tilde{f}_\eps(\cdot\ ;z)]$ and $\tilde{R}_{S,\eps}(\cdot) = \frac{1}{n}\sum_{i=1}^n \tilde{f}_\eps(\cdot\ ;z)$.
Let $x = (w, \tau)$, where $\tau \in \argmin_{\tau} \E_{z\sim\caD}[f(w, \tau; z)]$ is the optimal threshold in the definition of CVaR with respect to $\caD$.

The first step is to apply the analysis of SGD~\cite{Hardt2016} to the smoothed auxiliary function $\tilde{f}_\eps(\cdot\ ; z)$.

\begin{lemma}\label{lem:opterr}
Suppose that $\eta \leq 2/(\beta + 2G^2/\eps)$.
Let $x_t$ be the iterate of Algorithm~\ref{alg:batch-SGD} and let $x = \frac{1}{T}\sum_{t=1}^T x_t$. Then,
\begin{align*}
&\E[\tilde{R}_\eps(x) - \min_{x^*}\tilde{R}_{S,\eps}(x^*)]
\leq \frac{1}{2}\left[ \frac{D^2 + 1}{\eta T} + \eta G_\alpha \left(1 + \frac{2T}{n} \right) \right].
\end{align*}
\end{lemma}

Then, by the previous lemma,
\begin{align*}
    &\E[R(x)]
    \leq \E[\tilde{R}_\eps(x)]
    \leq \E[\min_{x^*}\tilde{R}_{S,\eps}(x^*)]
    + \frac{1}{2}\left[ \frac{D^2 + 1}{\eta T} + \eta G_\alpha \left(1 + \frac{2T}{n} \right) \right]  \\
    &\leq \E[\min_{x^*}R_{S}(x^*)]
    + \frac{1}{2}\left[ \frac{D^2 + 1}{\eta T} + \eta G_\alpha \left(1 + \frac{2T}{n} \right) \right] + \eps.
\end{align*}

Now, we optimize $\eta$ and $\eps$.
Note that $\eta$ and $\eps$ must satisfy $\eta\leq 2\alpha/(\beta+2G^2/\eps)$ to apply Lemma~\ref{lem:opterr}.
The following lemma formalizes the parameter tuning.

\begin{lemma}\label{lem:paramtuning}
Suppose that we run Algorithm~\ref{alg:batch-SGD} for $T = c n$ with $c > 1$.
If $c$ is sufficiently large so that~\eqref{eq:cond-on-c} holds, then we can choose $\eta$ and $\eps$ satisfying $\eta\leq \frac{2\alpha}{\beta+2G^2/\eps}$ and
\begin{align*}
&\frac{1}{2}\left[ \frac{D^2 + 1}{\eta T} + \eta G_\alpha \left(1 + \frac{2T}{n} \right) \right] + \eps
\leq \frac{G_\alpha \sqrt{D^2+1}}{\sqrt{n}} \left(\sqrt{\frac{1+2c}{c}} + \frac{2}{\sqrt{c(1+2c)}} \right).
\end{align*}
\end{lemma}

Finally, the desired bound follows from
\begin{align*}
R(x) &=\min_{\tau} \E_{z\sim\caD}[f(w,\tau; z)] = \CVaR_\caD(w) \text{ and} \\
\min_{x^*} R_S(x^*) &\leq \min_{w^*}\CVaR_S(w^*). \qedhere
\end{align*}

\subsubsection{Minibatch SGD}
A well-known common practice in SGD is that rather than using a gradient estimator computed from one sample, we use a gradient estimator averaged in a minibatch.
Further, we analyze this variant of SGD (Algorithm~\ref{alg:minibatch-SGD}).
Let $b$ be the size of a minibatch.

\begin{algorithm}[h!]
  \caption{Smoothed Stochastic Gradient Descent for CVaR with Minibatch}\label{alg:minibatch-SGD}
  \begin{algorithmic}[1]
    \REQUIRE Samples $S = \{z_1,\dots,z_n\}$, the number of iterations $T$, and minibatch size $b$.
    \STATE Initialize $x_1 = (w_1,\tau_1) \in K \times [0,1]$ arbitrarily.
    \STATE Set $\eta \gets \frac{b\sqrt{D^2+1} \sqrt{n}}{G_\alpha \sqrt{T(n+2T)}}$ and $\eps \gets 2G_\alpha\eta$.
    \FOR{$t=1,\dots,T$}
      \STATE For $i = 1, \dots, b$, sample $z_i \sim S$ and compute $g_i \in \partial_x \tilde{f}_\eps(x_t;z_i)$.
      \STATE Let $g_t = \frac{1}{b}\sum_{i=1}^b g_i$ and update $x_{t+1} = \proj_K(x_t - \eta g_t)$.
    \ENDFOR
    \RETURN $w = \frac{1}{T}\sum_{t=1}^{T} w_t$
  \end{algorithmic}
\end{algorithm}

\begin{theorem}\label{thm:minibatch-SGD}
Assume that the same assumption as in Theorem~\ref{thm:convex-SGD} holds and  $\ell(\cdot\ ;z)$ is $\beta$-smooth for all $z$.
Suppose that we run Algorithm~\ref{alg:minibatch-SGD} over a set $S$ of $n$ samples with $T=cn$ for $c > 0$, and let $w$ be the average of $T$ iterates.
If $c$ is sufficiently large so that
\begin{align}\label{eq:cond-on-c-minibatch}
\frac{b\sqrt{D^2+1}}{G_\alpha \sqrt{c(1+2c)n}} \leq \frac{\alpha}{\beta},
\end{align}
then
\begin{align*}
  & \E_{S,w}[\CVaR_\caD(w) - \min_{w^*\in K}\CVaR_S(w^*)]
  \leq \frac{G_\alpha \sqrt{D^2+1}}{\sqrt{bn}} \left(\sqrt{\frac{1+2c}{c}} + \frac{2}{\sqrt{1+2c}} \right).
\end{align*}
where the expectation is taken over $S$ and the randomness in the algorithm.
\end{theorem}
We defer the proof to Appendix due to space limitations.


\section{Nonconvex Loss}\label{sec:nonconvex}
We show that, even when the loss function is not convex, online gradient descent generalizes with respect to CVaR.
The algorithm is presented in Algorithm~\ref{alg:nonconvex-smoothed-SGD}.

\begin{algorithm}
\caption{Smoothed Online Gradient Descent for Nonconvex CVaR}\label{alg:nonconvex-smoothed-SGD}
\begin{algorithmic}[1]
    \REQUIRE The number of iterates $n$.
    \STATE Initialize $x_1 = (w_1,\tau_1) \in \bbR^d \times [0,1]$ arbitrarily.
    \STATE $\eps \gets G_\alpha^{2/3}G^{2/3}n^{-1/6}$, $\eta \gets  \frac{\alpha}{(\beta+G^2/\eps)G_\alpha^2\sqrt{n}}$.
    \FOR{$i=1,\dots,n$}
        \STATE Observe a sample $z_i$.
        \STATE Compute $g_i = \nabla_x \tilde{f_\eps}(x_i; z_i)$ and update $x_{i+1} = x_i - \eta g_i$.
    \ENDFOR
    \RETURN $(w,\tau) = (w_s,\tau_s)$, where $s$ is uniformly random over $[n]$.
\end{algorithmic}
\end{algorithm}

\begin{theorem}
  Assume that the feasible region $K$ is $\bbR^d$ (i.e., unconstrained) and the loss function $\ell(\cdot\ ;z)$ is $G$-Lipschitz and $\beta$-smooth for some $G,\beta>0$, and has a range $[0,1]$ for all $z$.
  Then, Algorithm~\ref{alg:nonconvex-smoothed-SGD} outputs $(w,\tau)$ such that
  \begin{align*}
    &\E_{w,\tau}\left[ \E_z [f(w, \tau; z)] - \CVaR_\caD(w) \right]
    \leq O\left( \frac{G_\alpha\beta^{1/2}}{n^{1/4}} + \frac{G_\alpha^{2/3}G^{2/3}}{n^{1/6}} \right).
  \end{align*}
\end{theorem}

\begin{proof}
  Let $(w_i, \tau_i)$ be the iterate of the algorithm for $i=1,\dots, n$.
  By standard analysis of the online gradient descent for nonconvex smooth functions (e.g., see~\citet[Appendix~B]{Allen2018}), we have
  \[
    \E_{i\sim [n]}\E_z[\norm{\nabla \tilde{f_\eps}(w_i, \tau_i; z)}_2^2] \leq O\left(\frac{G_\alpha^2(\beta+G^2/\eps)}{\alpha\sqrt n} \right).
  \]
  Let us define $\tau_i^* \in \argmin_\tau \E_z[\tilde{f_\eps}(w_i,\tau;z)]$ for $i\in [n]$.
  Now, because $\tilde{f_\eps}(w, \tau; z)$ is convex in $\tau$, we have
  \begin{align*}
    & \E_{i\sim[n]} \left[\E_z[\tilde{f_\eps}(w_i, \tau_i;z)] - \CVaRtilde_{\caD,\eps}(w_i) \right]
    = \E_{i,z}[\tilde{f_\eps}(w_i, \tau_i;z) - \tilde{f_\eps}(w_i, \tau^*;z)] \\
    &\leq \E_{i,z}[\nabla_\tau \tilde{f_\eps}(w_i, \tau_i;z) (\tau_i-\tau^*)] \tag{by convexity}\\
    &\leq \sqrt{\E_{i,z}[(\nabla_\tau \tilde{f_\eps}(w_i, \tau_i;z))^2]} \sqrt{\E_i[(\tau_i-\tau^*)^2]} \tag{by Cauchy-Schwartz}\\
    &\leq \sqrt{\E_{i,z}[\norm{\nabla \tilde{f_\eps}(w_i, \tau_i;z)}_2^2]}
    = O\left(\frac{G_\alpha{(\beta + G^2/\eps)}^{1/2}}{\alpha^{1/2}n^{1/4}}\right)
    = O\left(\frac{G_\alpha\beta^{1/2}}{\alpha^{1/2}n^{1/4}} + \frac{ G_\alpha G}{\eps^{1/2}\alpha^{1/2}n^{1/4}} \right).
  \end{align*}
  Then, omitting the $\alpha^{-1/2}$ factor, we have
  \begin{align*}
    & \E_{i\sim[n]}\left[ \E_z[f(w_i, \tau_i;z)] - \CVaR_\caD(w_i) \right] \\
    &= \E_{i,z}[\tilde{f_\eps}(w_i, \tau_i;z) - \CVaRtilde_{\caD,\eps}(w_i)]
    + \E_{i,z}[f(w_i, \tau_i;z) -\tilde{f_\eps}(w_i, \tau_i;z)]
    +  \E_i[\CVaRtilde_{\caD,\eps}(w_i) - \CVaR_\caD(w_i)] \\
    & = O\left(\frac{G_\alpha\beta^{1/2}}{n^{1/4}} + \frac{G_\alpha G}{\eps^{1/2} n^{1/4}}  + \eps\right).
  \end{align*}
Setting $\eps = G_\alpha^{2/3}G^{2/3}n^{-1/6}$ completes the proof.
\end{proof}


\begin{figure*}[t!]
  \centering
  \subfigure[CVaR$_{0.05}$]{\includegraphics[width=.425\hsize]{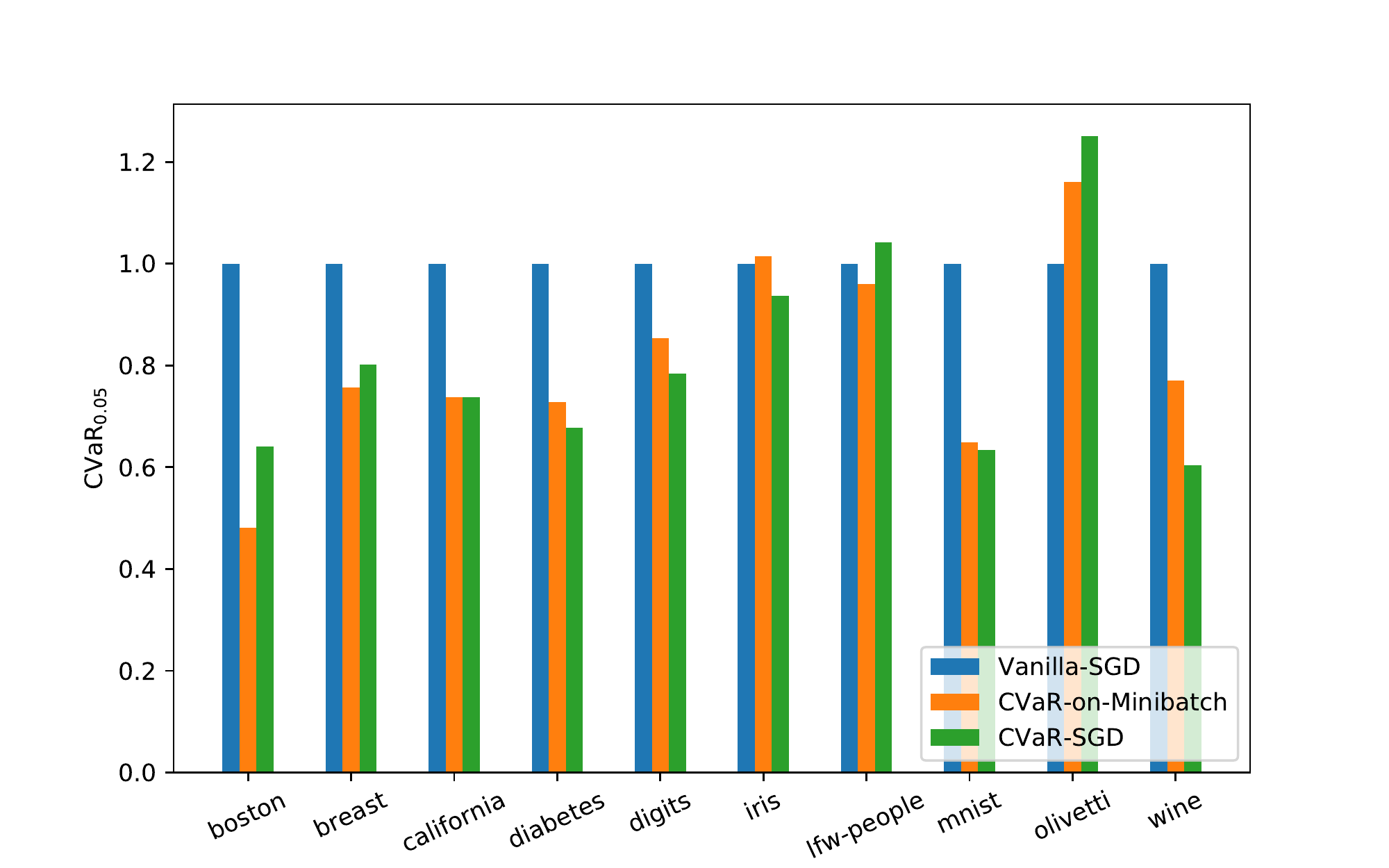}}
  \subfigure[CVaR$_{0.1}$]{\includegraphics[width=.425\hsize]{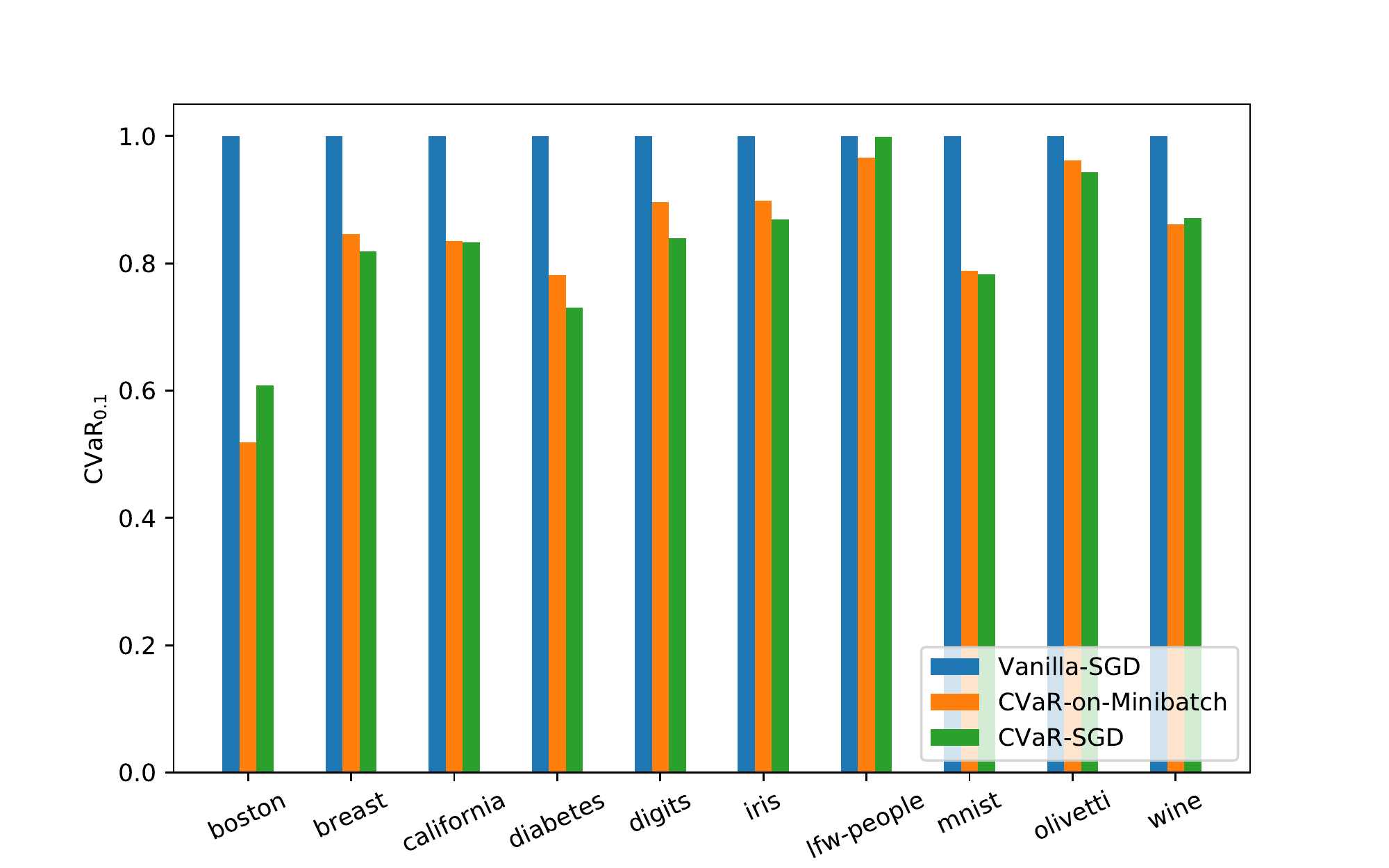}}
  \subfigure[Accuracy]{\includegraphics[width=.425\hsize]{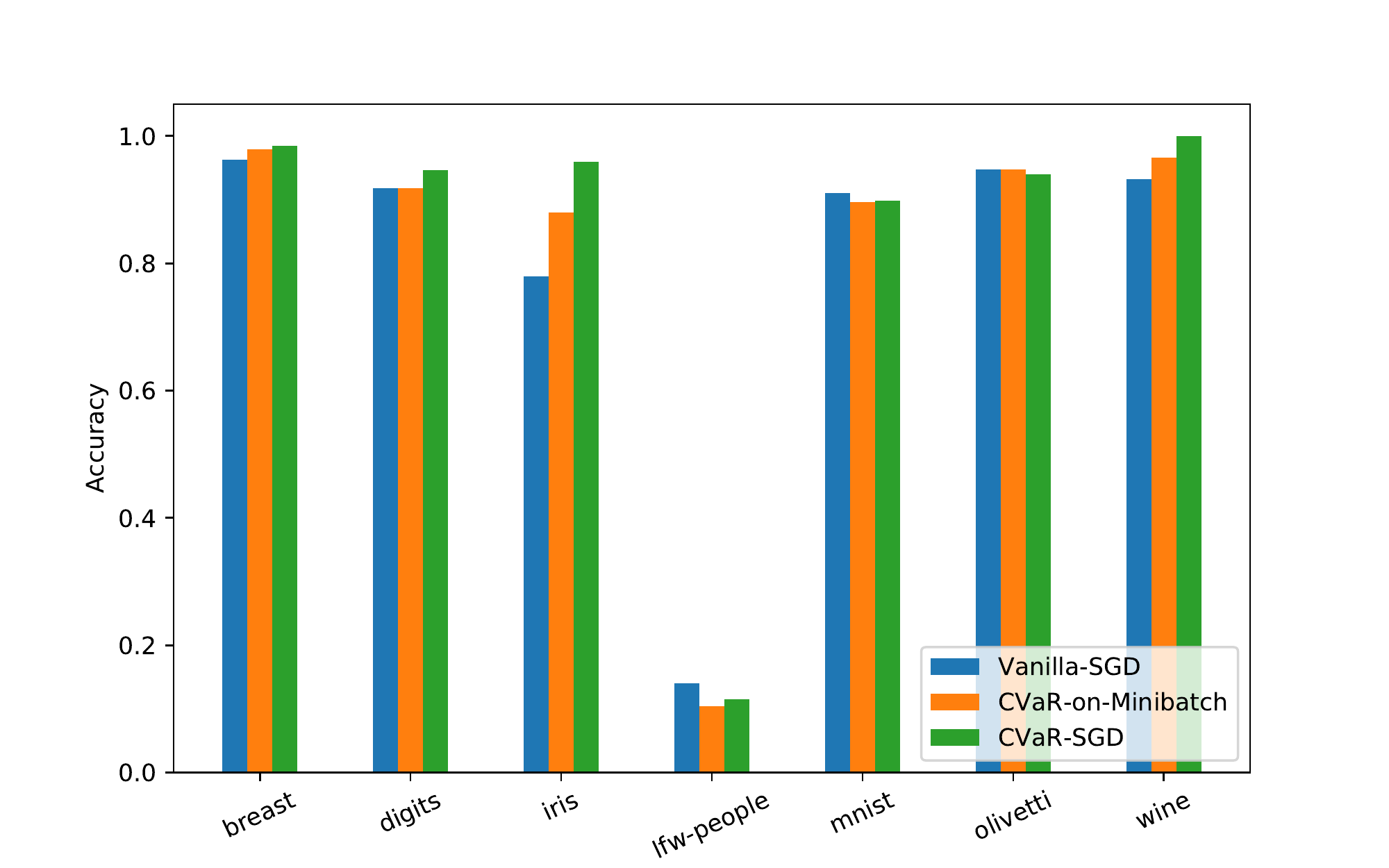}}
  \subfigure[Loss]{\includegraphics[width=.425\hsize]{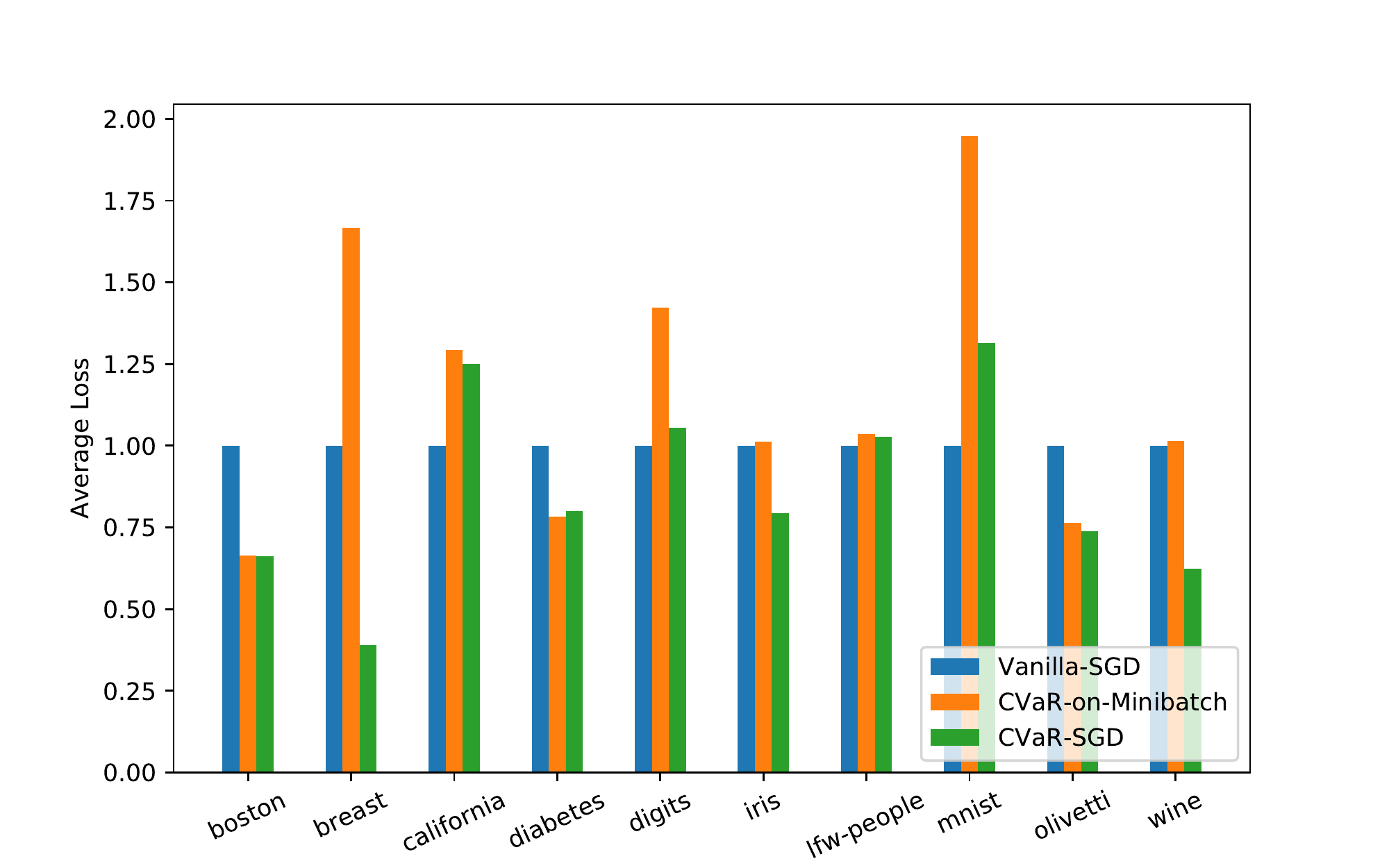}}
  \caption{Summary of the experimental results for linear models. For comparison, the CVaR$_{0.05}$, CVaR$_{0.05}$, and average loss of \textsf{Vanilla-SGD} are normalized to one.}\label{fig:summary}
\end{figure*}

\begin{figure*}[t!]
  \centering
  \subfigure[CVaR$_{0.1}$]{\includegraphics[width=.32\hsize]{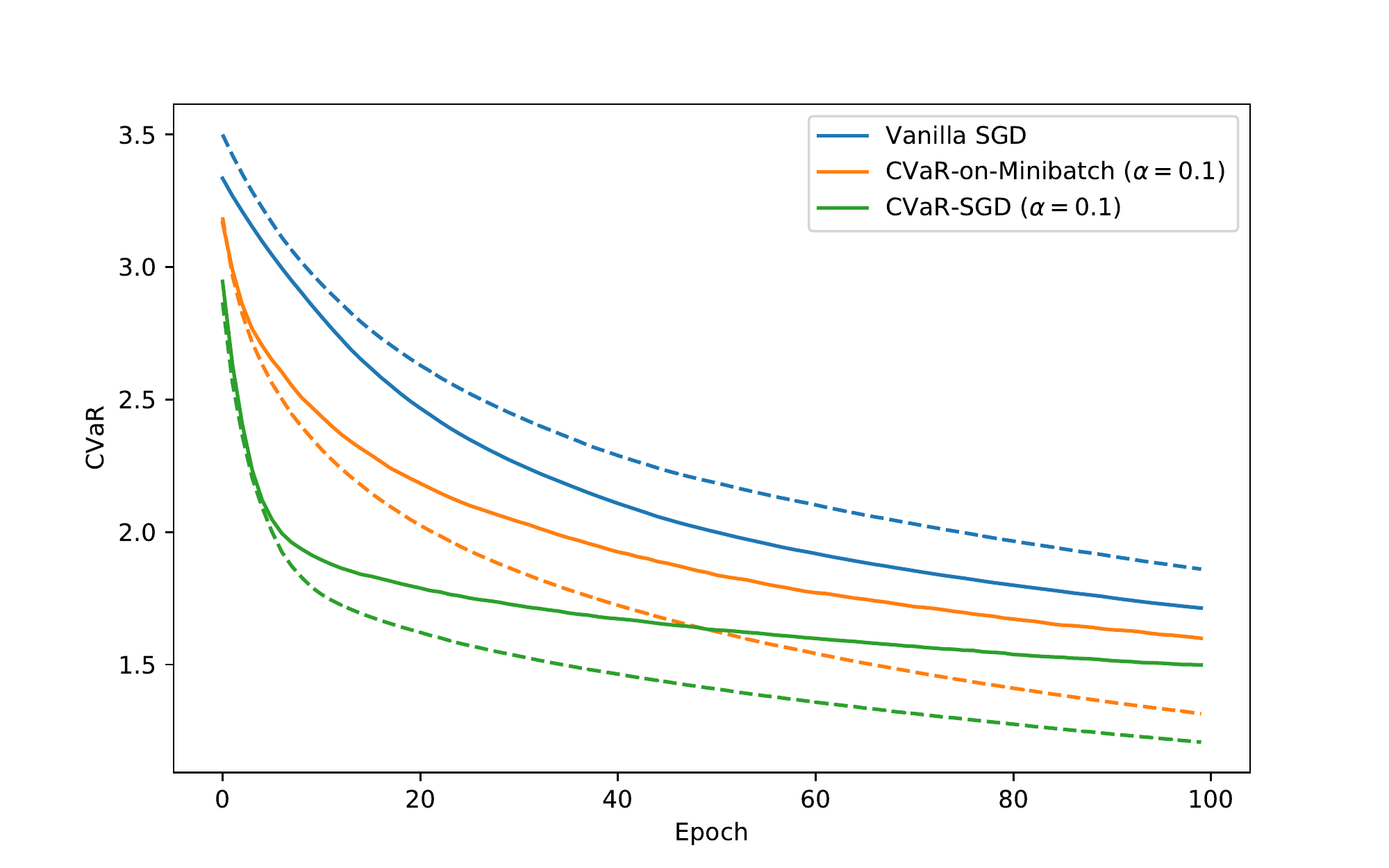}}
  \subfigure[Accuracy]{\includegraphics[width=.32\hsize]{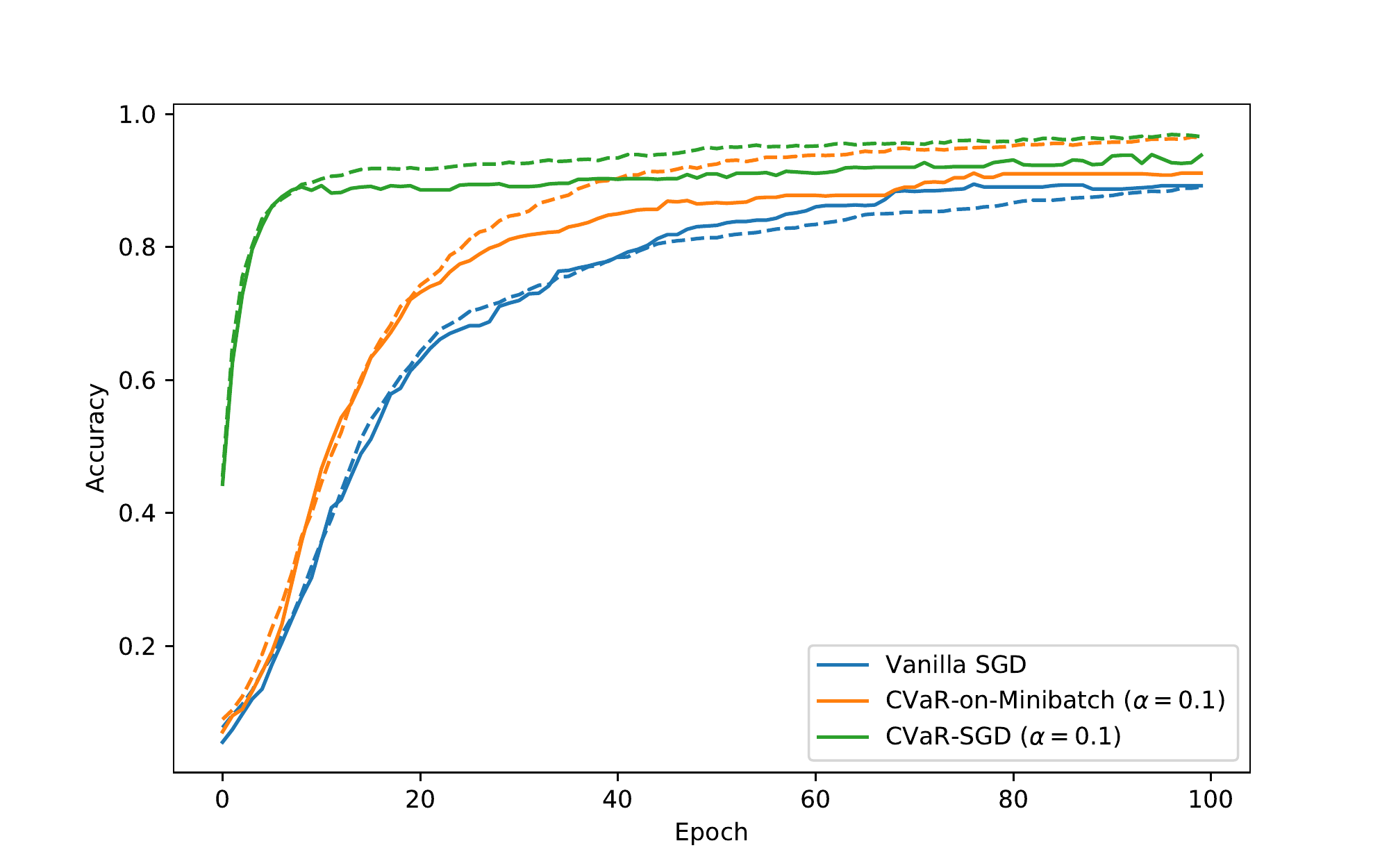}}
  \subfigure[Loss]{\includegraphics[width=.32\hsize]{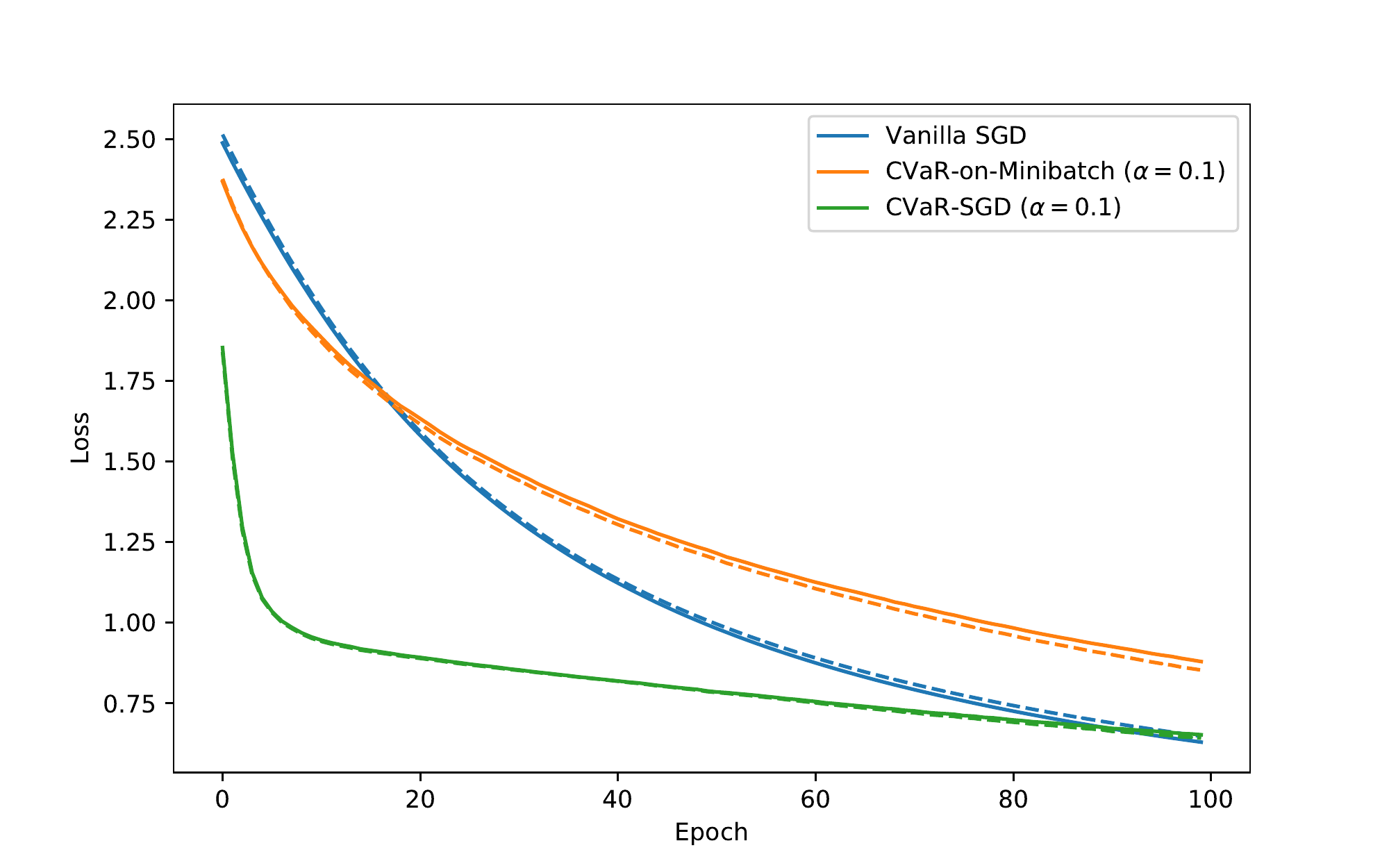}}
  \caption{Transition of the CVaR, accuracy, and average loss of the linear model on the digits dataset. Solid and dashed lines represent the results on the validation and training data, respectively.}\label{fig:digits-linear-curve}
\end{figure*}

\begin{figure*}[t!]
  \centering
  \subfigure[$\alpha=0.05$]{\includegraphics[width=.48\hsize]{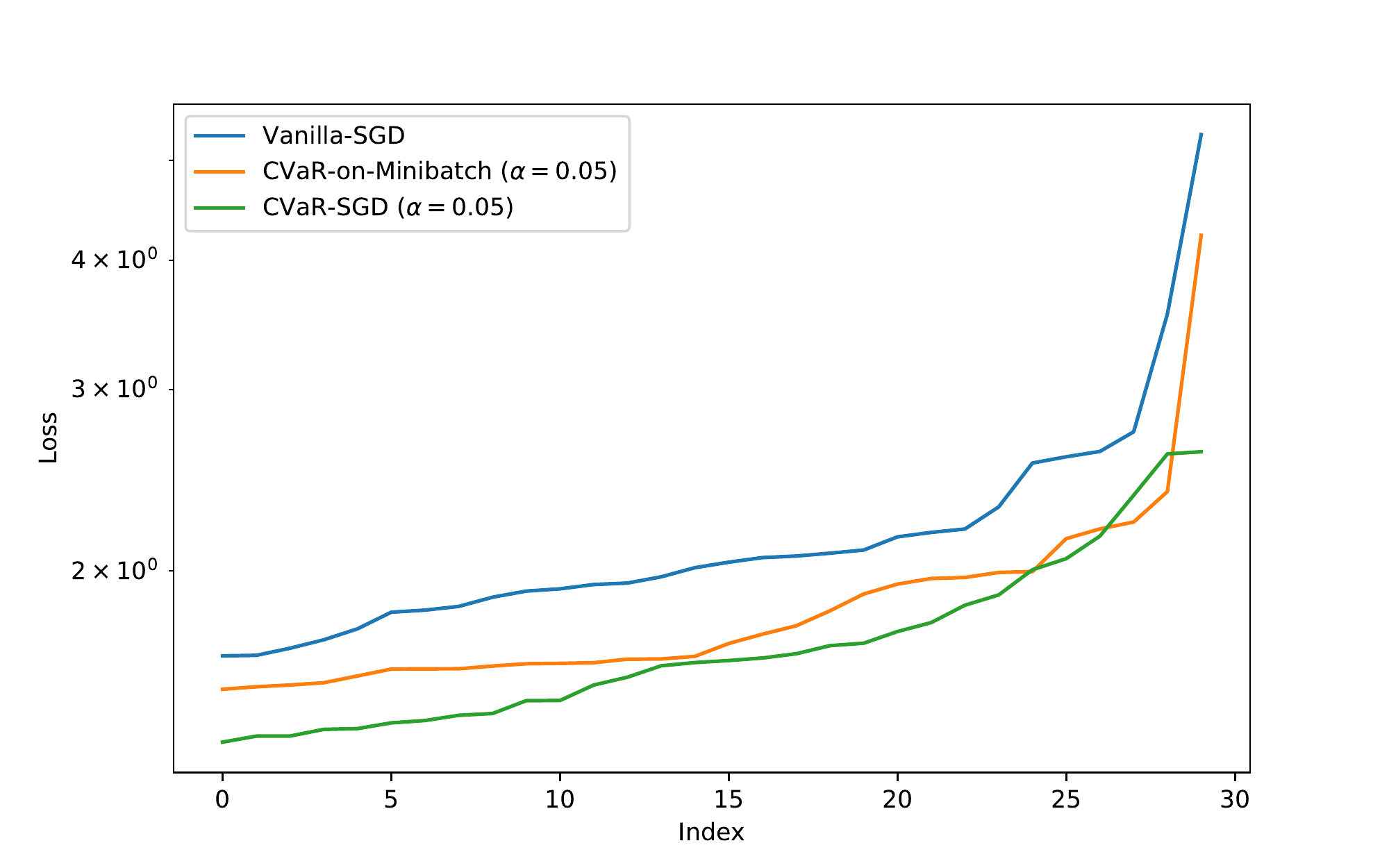}}
  \subfigure[$\alpha=0.1$]{\includegraphics[width=.48\hsize]{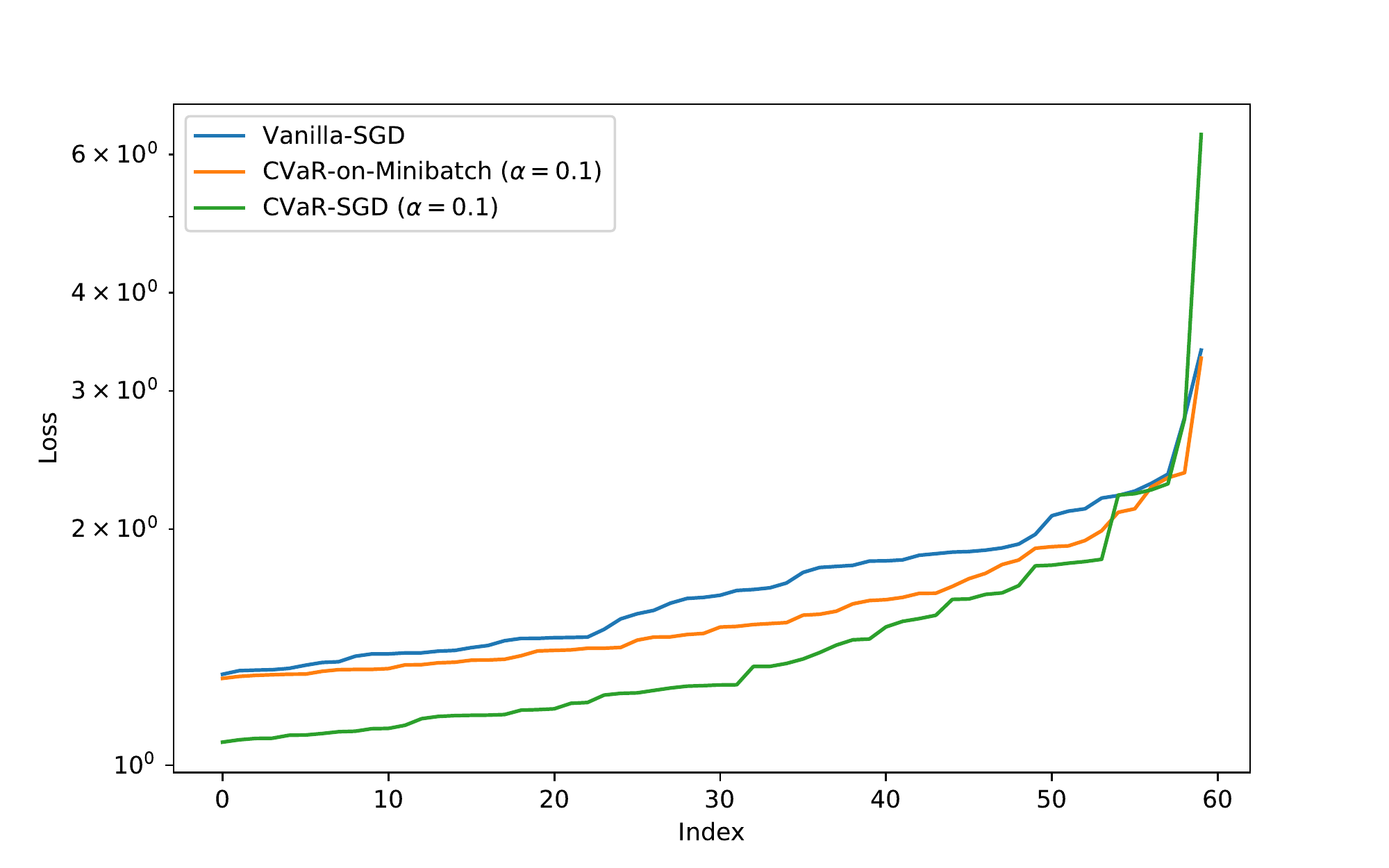}}
  \caption{Top $\alpha$-fraction of loss values of the linear model on the digits dataset sorted in the increasing order.}\label{fig:digits-linear-histogram}
\end{figure*}

\section{Experiments}\label{sec:experiments}

In this section, we demonstrate that CVaRs attained by our algorithm outperform those attained by baseline methods.

\paragraph{Models and Datasets}
We solved (multinomial) logistic regression for classification tasks and linear regression for regression tasks using datasets provided in the scikit-learn library~\cite{scikit-learn} and the MNIST dataset~\cite{Lecun1998}.
For each dataset used from the scikit-learn library, we randomly split data into training validation data such that the former has $2/3$ of the examples.
In Appendix~\ref{sec:further-experimental-results}, we also show experimental results regarding nonlinear neural networks.

\paragraph{Methods}
We compared the following three algorithms, all of which are based on SGD\@.
\begin{itemize}
\itemsep=0pt
\item \textsf{Vanilla-SGD}: The vanilla SGD, where we move along the (negative) gradient of the average of the loss functions in the current minibatch.
\item \textsf{CVaR-on-Minibatch}: A variant of \textsf{Vanilla-SGD}, where we move along the (negative) gradient of the CVaR, i.e., the average of the top $\alpha$-fraction of the loss functions in the current minibatch.
If the size of the minibatch is smaller than $1/\alpha$, then we use the gradient of the example with the maximum loss value.
\item \textsf{CVaR-SGD} (proposed): Algorithm~\ref{alg:minibatch-SGD}.
\end{itemize}
We fix the minibatch size to 512 and applied weight decay to stabilize the learning process.
As a preliminary experiment, for each dataset and method, we train the model with step sizes $0.001$, $0.005$, and $0.01$ and weight decaying factors $0$, $0.0001$, and $0.001$ for 100 epochs.
Among these hyperparameters, we used the best one with the smallest average loss on the validation data for the final plot.

\paragraph{Results}

Figure~\ref{fig:summary} illustrates the experimental results.
As expected, CVaRs obtained by \textsf{CVaR-SGD} are much smaller than those obtained by \textsf{Vanilla-SGD} and they are slightly smaller than those obtained by \textsf{CVaR-on-Minibatch}.

We observe a similar tendency in accuracy for classification tasks, which can be explained as follows:
We can correctly guess the label of an example if the loss for the example is sufficiently small.
Hence, to improve accuracy, it is important to train the model so that we have fewer examples with high losses, which is being attempted by \textsf{CVaR-SGD}.

Although \textsf{CVaR-SGD} does not attempt to minimize the (average) loss, the losses obtained by \textsf{CVaR-SGD} are comparable to those obtained by \textsf{Vanilla-SGD} for many tasks.

Figure~\ref{fig:digits-linear-curve} shows the transition of CVaR$_{0.1}$, accuracy, and average loss over epochs on the digits dataset.
We can observe that \textsf{CVaR-SGD} outperforms other methods for every criterion.

Figure~\ref{fig:digits-linear-histogram} shows the top $\alpha$-fraction of the losses sorted in the increasing order on the validation data of the digits dataset for $\alpha = 0.05,0.1$.
We can observe that \textsf{CVaR-SGD} successfully achieves smaller losses for the top $\alpha$-fraction of examples compared to other methods.


\section{Conclusions}\label{sec:conclusions}
We proposed a risk-averse statistical learning framework, where the performance of a learning algorithm is evaluated by the conditional value-at-risk (CVaR) of losses.
We devised algorithms based on stochastic gradient descent for this framework and provided a generalization bound on CVaR even when the loss functions are nonconvex.
By conducting numerical experiments on various machine learning tasks, we demonstrated that our algorithms can effectively minimize CVaR compared with other baseline algorithms.

\bibliography{main}

\appendix

\section{Basic Facts}
\begin{lemma}\label{lem:cocoercive}
    If $f : \bbR^d \to \bbR$ is $\beta$-smooth, then
    \[
        \inprod{\nabla f(x) - \nabla f(y), x - y} \geq \frac{1}{\beta} \norm{x-y}^2.
    \]
\end{lemma}

\section{Omitted Proofs}
\subsection{Proof of Lemma~\ref{lem:Lips_const}}
It suffices to show that $\norm{g} \leq \Galpha$ for all subgradients $g \in \partial_x f(x; z)$.
Let $x = (w,\tau)$. We need to consider three cases.
\paragraph{Case 1: $\ell(w;z) > \tau$} In this case, we have $\partial_x f(x; z) = \{[\nabla_w \ell(w; z)/ \alpha, 1-1/\alpha]^\top \}$, and hence, $\norm{g}^2 \leq  \frac{G^2}{\alpha^2} + (1-\frac{1}{\alpha})^2 = \frac{G^2+(\alpha-1)^2}{\alpha^2}$ for all $g \in \partial_x f(x; z)$.
\paragraph{Case 2: $\ell(w;z) = \tau$} In this case, we have
\[
\partial_x f(x; z) =  \{[t \nabla_w \ell(w; z)/\alpha, 1-t/\alpha]^\top : t \in [0,1] \}.
\]
Thus for all $g \in \partial_x (x; z)$, we have $\norm{g}^2 \leq \max_{t \in [0,1]} [ t^2\frac{G^2}{\alpha^2} + (1-\frac{t}{\alpha})^2 ] = \max\left\{\frac{G^2}{\alpha^2},\frac{(1-\alpha)^2}{\alpha^2} \right\}$.
\paragraph{Case 3: $\ell(w;z) < \tau$} In this case, we have $\partial_x (x; z) =  \{[\mathbf{0}, 1]^\top \}$, and hence, $\norm{g} = 1$ for all $g \in \partial_x (x; z)$.

This completes the proof.

\subsection{Proof of Lemma~\ref{lem:smoothing}}
For the simplicity of the exposition, we drop $z$ from the notations.
The first claim is immediate from the previous lemma.
For the second and third claims, note that
\[
\nabla \tilde{f}_\eps(x) =
\begin{bmatrix}
  \frac{1}{\alpha}\dot{\rho}_\eps(\ell(w) - \tau)\nabla \ell(w) \\
  -\frac{1}{\alpha}\dot{\rho}_\eps(\ell(w) - \tau) + 1
\end{bmatrix},
\]
where $\dot{\rho}_\eps$ denotes the derivative of $\rho_\eps$.
Since $\dot{\rho}_\eps \in [0, 1]$, we have $\norm{\nabla \tilde{f}_\eps}^2 \leq \max_{t\in[0,1]} [\frac{t^2G^2}{\alpha^2} + (1-\frac{t}{\alpha})^2] \leq G_\alpha^2$.

Finally, for $x=(w,\tau)$ and $x'=(w',\tau')$,
\begin{align*}
  &\norm{\dot\rho_\eps(\ell(w) - \tau)\nabla \ell(w) - \dot\rho_\eps(\ell(w') - \tau')\nabla \ell(w')} \\
  &\leq \abs{\dot\rho_\eps(\ell(w) - \tau)} \cdot \norm{\nabla \ell(w) - \nabla \ell(w')}
  + \norm{\nabla \ell(w')} \cdot \abs{\dot\rho_\eps(\ell(w) - \tau) - \dot\rho_\eps(\ell(w') - \tau')} \\
  &\leq \beta \norm{w} + G \cdot 2/\eps \abs{(\ell(w) - \tau) - (\ell(w') - \tau')} \\
  & \leq  \beta\norm{w-w'} + 2G^2/\eps \norm{x-x'} \\
  &= \left(\beta + 2G^2/\eps \right) \norm{x-x'}.
\end{align*}
We can bound the last coordinate of $\nabla \tilde{f}_\eps(x)$ similarly.

\subsection{Proof of Lemma~\ref{lem:smoothing-gap-CvaR}}
\begin{proof}
    By Lemma~\ref{lem:smoothing}, we have
    \[
        \E_z \tilde{f}_\eps(w, \tau ;z) \leq \E_z f(w,\tau ; z) + \eps
    \]
    for all $\tau$.
    Let us take $\tau^* \in \argmin_{\tau\in [0,1]} f(w,\tau; z)$ and we have
    \[
        \min_\tau \E_z \tilde{f}_\eps(w, \tau ;z) \leq \E_z \tilde{f}_\eps(w, \tau^* ;z) \leq \E_z f(w,\tau^* ; z) + \eps,
    \]
    which shows that $\CVaRtilde_{\caD,\eps}(w)\leq \CVaR_\caD(w)  + \eps$.
    The other direction is trivial.
\end{proof}

\subsection{Proof of Lemma~\ref{lem:opterr}}
We use the following analysis of SGD from \citet{Hardt2016}.

\begin{lemma}[\citet{Hardt2016}]
    Let $f(x;z)$ be a $L$-Lipschitz and $\gamma$-smooth convex function with the range bounded in $[0,1]$ for all $z$.
    Let $F$ be a convex set with diameter $P$.
    Let $S = \{z_1, \dots, z_n\}$ be samples and define
    $R(x) = \E_{z\sim \caD}[f(x;z)]$ and $R_S(x) = \frac{1}{T}\sum_{i=1}^n f(x;z_i)$.
    Let $x_t$ be the SGD iterate with $n$ samples and a learning rate $\eta$, and $x := \frac{1}{T}\sum_{t=1}^T x_t$.
    If $\eta \leq 2/\gamma$, then
    \begin{align*}
        \E[R(x) - \min_{x^* \in F}R_S(x^*) ] \leq \frac{1}{2}\left[ \frac{P^2}{\eta T} + \eta L^2 \left( 1 + \frac{2T}{n} \right) \right].
    \end{align*}
\end{lemma}

Applying this lemma to our setting, we obtain
\begin{align*}
    \E[R(x) - \min_{x^* \in F}R_S(x^*) ]
    &\leq \frac{1}{2}\left[ \frac{D^2 + 1}{\eta T} + \eta G_\alpha^2 \left(1 + \frac{2T}{n} \right) \right].
\end{align*}
Note that in our setting $L = G_\alpha$, $\gamma = \betaalpeps$, and $P = \sqrt{D^2+1}$.

\subsection{Proof of Lemma~\ref{lem:paramtuning}}
Let us set
\begin{align*}
    \eta &= \frac{\sqrt{D^2+1} \sqrt{n}}{G_\alpha \sqrt{T(n+2T)}} = \frac{\sqrt{D^2+1}} {G_\alpha \sqrt{n}} \cdot \frac{1}{\sqrt{1+2c}},  \\
    \eps &= 2G_\alpha^2 \eta = \frac{2G_\alpha \sqrt{D^2+1}}{\sqrt{n}} \cdot \frac{1}{\sqrt{1+2c}}.
\end{align*}
We must check that this choice satisfies $\eta \leq \frac{2\alpha}{\beta + 2G^2/\eps}$.
To achieve this, first note that
\begin{align*}
        \frac{2\alpha}{\beta + 2G^2/\eps} \geq \frac{\alpha}{\max\{\beta,2G^2/\eps \}} = \min\left\{\frac{\alpha}{\beta}, \frac{\alpha\eps}{2G^2} \right\}.
\end{align*}
Hence it suffices to check that $\eta \leq \alpha/\beta$ and $\eta \leq \alpha\eps/2G^2$.
The former condition is satisfied by our assumption \eqref{eq:cond-on-c} on $c$.
The latter condition follows from
\begin{align*}
    \eta = \frac{\eps}{2G_\alpha^2} \leq \frac{\alpha^2\eps}{2G^2} \leq \frac{\alpha \eps}{2G^2},
\end{align*}
because
\[
    G_\alpha = \Galpha \geq \frac{G}{\alpha}.
\]
Now by the choice of $\eta$ and $\eps$, we have
\begin{align*}
\frac{1}{2}\left[ \frac{D^2 + 1}{\eta T} + \eta G_\alpha^2 \left(1 + \frac{2T}{n} \right) \right] + \eps
&\leq \frac{G_\alpha \sqrt{D^2+1}}{\sqrt{n}} \left(\sqrt{\frac{1+2c}{c}} + \frac{2}{\sqrt{1+2c}} \right).
\end{align*}

\subsection{Proof of Theorem~\ref{thm:minibatch-SGD}}
For the proof, we use the concept called \emph{uniform stability}.

\begin{definition}[uniform stability]
    Let $f(x;z)$ be a real-valued function parametrized by $z$.
    A randomized algorithm $A$ is \emph{$\epsilon$-uniformly stable} with respect to $f$ if for any two sequences of examples $S$ and $S'$ that differ in at most one example, we have
    \[
        \sup_{z}\E_A[f(A(S);z)-f(A(S');z)] \leq \eps.
    \]
\end{definition}

Let us denote $R(x) = \E_{z\sim \caD}[f(x;z)]$ and $R_S(x) = \E_{z \sim S}[f(x;z)]$.

\begin{lemma}[{\citet[Theorem~2.2]{Hardt2016}}]\label{lem:stability-generalization}
    An $\eps$-uniformly stable algorithm $A$ satisfies
    \[
        \abs*{
        \E_{S,A}\left[ R_S(A(S)) - R(A(S)) \right]
        }\leq \eps.
    \]
\end{lemma}
First, we analyze the stability of minibatch SGD.
\begin{lemma}
    Assume that $f(\cdot;z)$ is $L$-Lipschitz and $\gamma$-smooth for all $z$.
    Let us consider the minibatch SGD iteration $x_{t+1} = \proj_K(x_t - \eta g_t)$ over samples $S=\{z_1, \dots, z_n\}$, where $g_t$ is the averaged subgradient estimate over minibatch of size $b$.
    Define $A(S) = \frac{1}{T}\sum_{t=1}^T x_t$.
    If $\eta \leq 2/\gamma$, then $A$ is $\eps$-uniformly stable where
    \[
        \eps \leq \eta\frac{L^2T}{n}.
    \]
\end{lemma}
\begin{proof}
    The proof is a simple modification of \citet{Hardt2016}.
    Without loss of generality, we can ignore the projection in SGD because the projection preserves uniform stability~\citep[Lemma~4.6]{Hardt2016}.
    Let $S$ and $S'$ be datasets differing in one element, and $x_t$ and $x'_t$ be iterates of minibatch SGD with $S$ and $S'$, respectively.
    Define $\delta_t = \norm{w_t - w_{t-1}}$ for each $t$.
    For each $t$, denote by $B_t$ and $B_t'$ the minibatches selected by $A$ with input $S$ and $S'$, respectively.
    Note that $\Pr(B_t = B_t') = (1-1/n)^b$.

    \paragraph{Case 1: $B_t = B_t'$} In this case, we have
    \begin{align*}
        \norm{x_{t+1} - x_{t+1}'}^2 &= \norm{x_t - x_t'}^2 - 2\eta\inprod{g_t - g_t', w_t - w_t'} + \eta^2\norm{g_t-g_t'}^2 \\
        &\leq \norm{x_t - x_t'}^2 -\left (\frac{2}{\gamma} - \eta\right)\eta\norm{g_t-g_t'}^2 \tag{by Lemma~\ref{lem:cocoercive}} \\
        &\leq \norm{x_t - x_t'}^2. \tag{since $\eta \leq 2/\gamma$}
    \end{align*}
    Hence $\delta_{t+1}\leq \delta_t$.

    \paragraph{Case 2: $B_t \neq B_t'$} In this case, we have
    \begin{align*}
        \norm{x_{t+1} - x_{t+1}'} &\leq \norm{x_{t+1} - x_t} + \norm{x'_{t+1} - x'_t} + \norm{x_t - x_t'}
        = \eta (\norm{g_t} + \norm{g_t'}) + \delta_t
        \leq 2\eta \frac{L}{b} + \delta_t.
    \end{align*}
    Therefore, we have
    \begin{align*}
        \E_A[\delta_{t+1} \mid \delta_t] &\leq \left(1-\frac{1}{n}\right)^b \delta_t + \left[ 1 -  \left(1-\frac{1}{n}\right)^b \right] \cdot (2\eta L + \delta_t) \\
        &\leq \delta_t + \frac{2\eta L}{b} \left[ 1 -  \left(1-\frac{1}{n}\right)^b \right] \\
        &\leq \delta_t + \frac{2\eta L}{n},
    \end{align*}
    where in the last inequality we used an elementary inequality $1-bx \leq (1-x)^b$ for $x \in [0,1]$.
    This yields $\E_A[\delta_T] \leq \frac{2\eta TL}{n}$.
    Since $f$ is $L$-Lipchitz,
    $\E_A[\abs{f(w_T;z)-f(w'_T;z)}] \leq L \E_A[\delta_t] \leq \frac{2\eta TL^2}{n}$.

    Now if we consider averaged iterates, one can remove a factor of 2 (see~\citet[Theorem~4.7]{Hardt2016}).
\end{proof}
\begin{lemma}[see e.g., {\citet[Theorem~3.4]{Hazan2016OCO}}]
    Assume that $f_i(x; z_i)$ is $L$-Lipschitz for $i=1,\dots,n$ and $\diam(F)\leq P$.
    Suppose that we run SGD on $R_S(x) = \frac{1}{n}\sum_{i=1}^n f_i(x; z_i)$ with a constant step size $\eta$.
    Then the averaged iterates $\bar{x}_T$ satisfies
    \[
        \E[R_S(\bar{x}_T) - \min_{x^*\in F}R_S(x^*)] \leq \frac{\eta L^2}{2} + \frac{D^2}{\eta T}.
    \]
\end{lemma}

Now we prove Theorem~\ref{thm:minibatch-SGD}.
Using the above lemmas for $\tilde{f}$, we have
\begin{align*}
    \E[\tilde{R}(x)] &\leq \E[\tilde{R}_S(x)] + \eta \frac{L^2T}{n} \\
    &\leq \E[\min_{x^*} \tilde{R}_S (x^*)] + \frac{\eta L^2}{2} + \frac{P^2}{\eta T} + \eta\frac{L^2T}{n}.
\end{align*}
Substituting $P = \sqrt{D^2 + 1}$ and $L = G_\alpha/\sqrt{b}$, we have
\begin{align*}
    \E[\tilde{R}(x) - \min_{x^*} \tilde{R}_S (x^*)]
    &\leq \frac{\eta G_\alpha}{2b} + \frac{D^2+1}{\eta T} + \eta\frac{G_\alpha T}{bn}
    = \eta\frac{G_\alpha}{2b}\left(1 + \frac{2}{T} \right) + \frac{D^2+1}{\eta T}.
\end{align*}
The rest is same as in Theorem~\ref{thm:convex-smooth-SGD}.

\begin{figure*}[t!]
  \centering
  \subfigure[CVaR$_{0.05}$]{\includegraphics[width=.45\hsize]{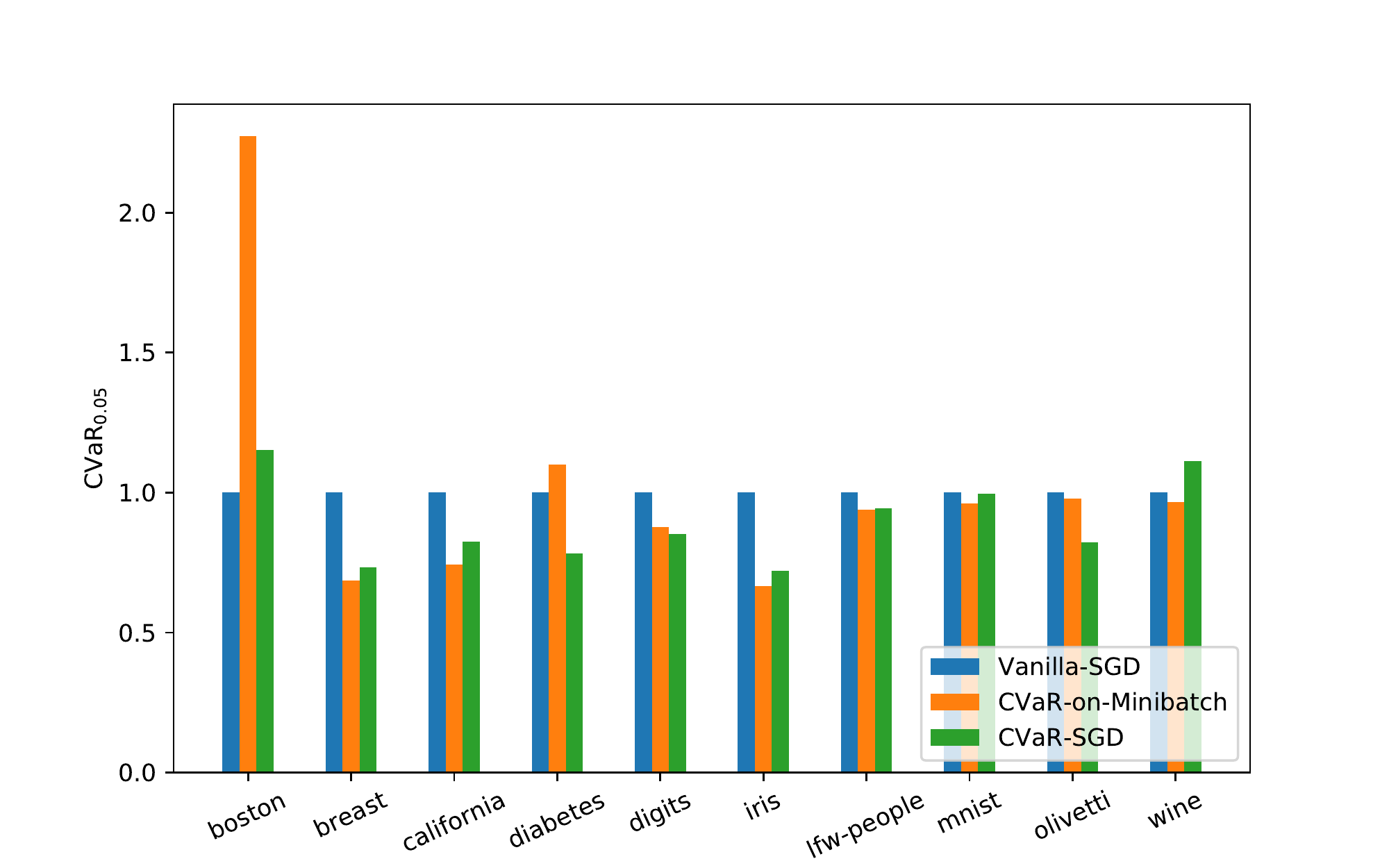}}
  \subfigure[CVaR$_{0.1}$]{\includegraphics[width=.45\hsize]{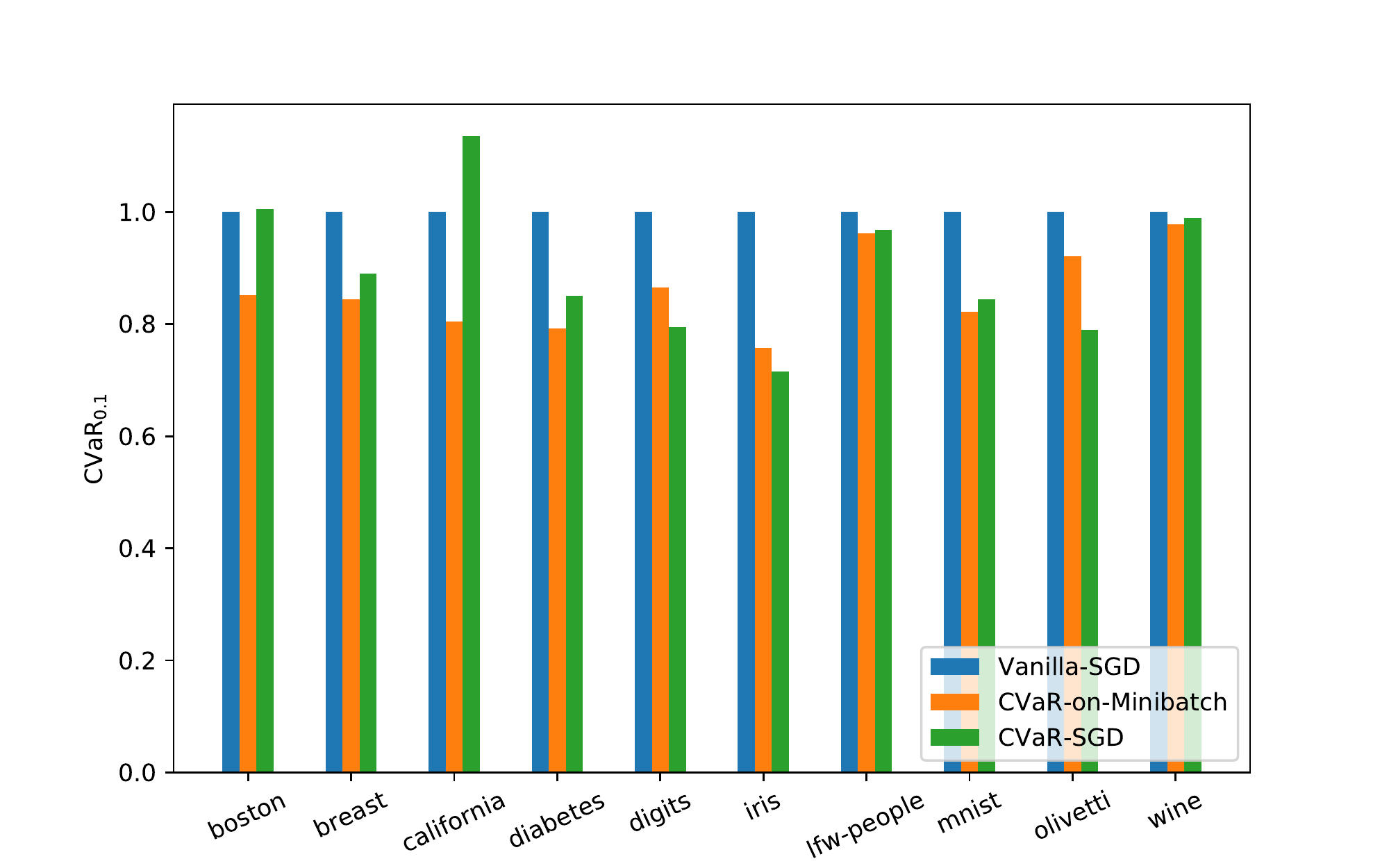}}
  \subfigure[Accuracy]{\includegraphics[width=.45\hsize]{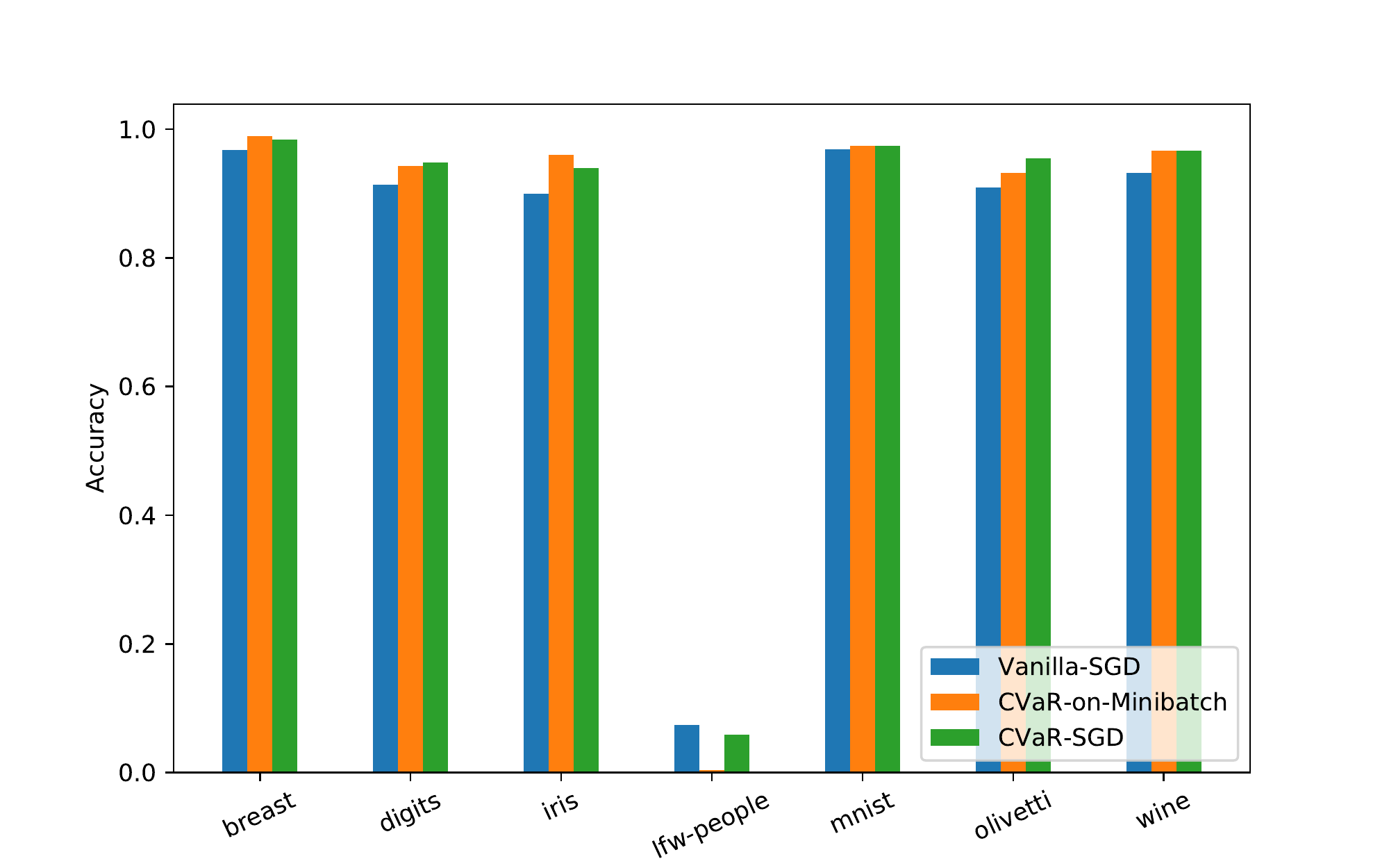}}
  \subfigure[Loss]{\includegraphics[width=.45\hsize]{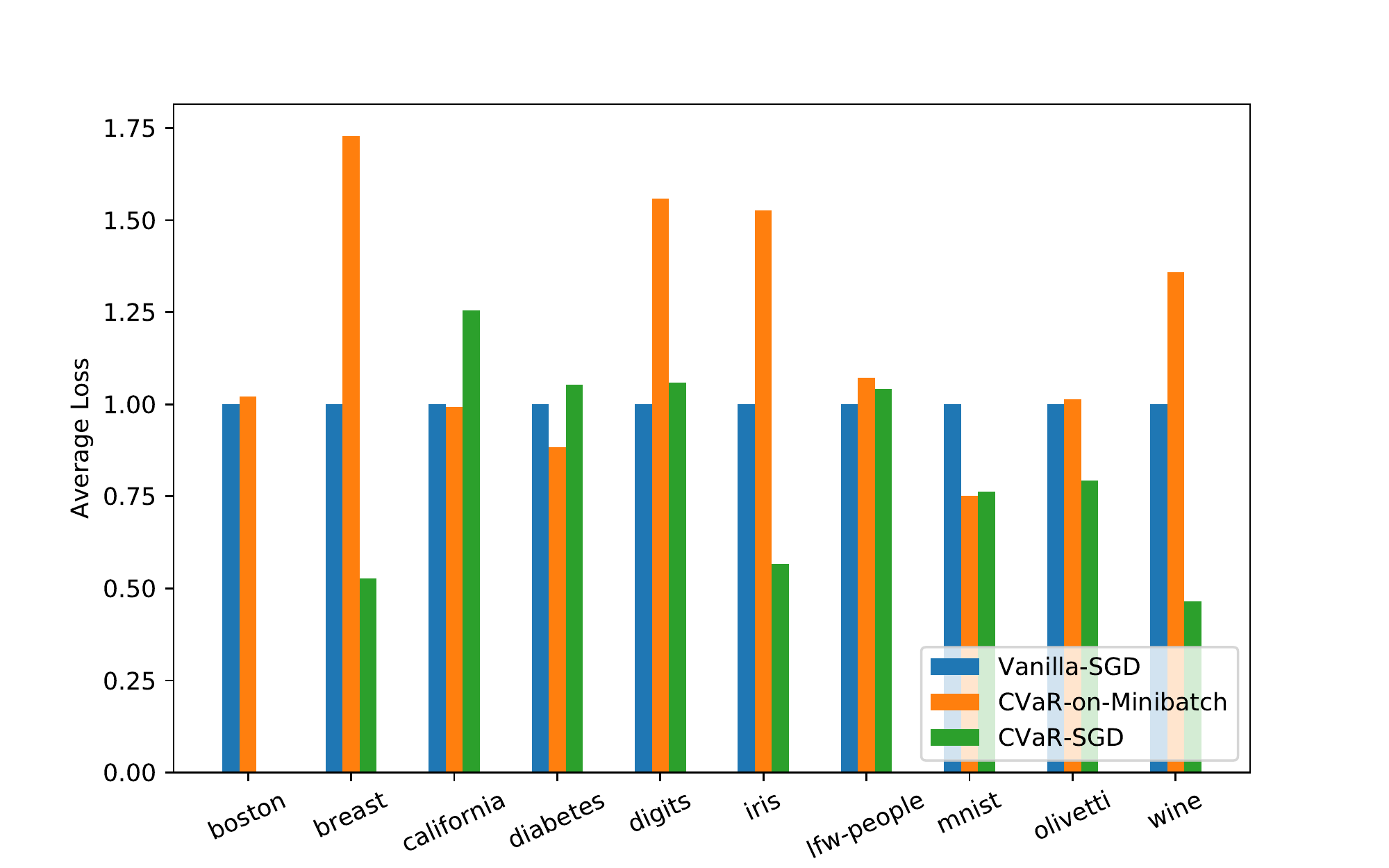}}
  \caption{Summary of the experimental results for three-layer neural networks. For comparison, the CVaR$_{0.05}$, CVaR$_{0.05}$, and average loss of \textsf{Vanilla-SGD} are normalized to one.}\label{fig:summary-3-layer}
\end{figure*}

\begin{figure*}[t!]
  \centering
  \subfigure[CVaR$_{0.1}$]{\includegraphics[width=.32\hsize]{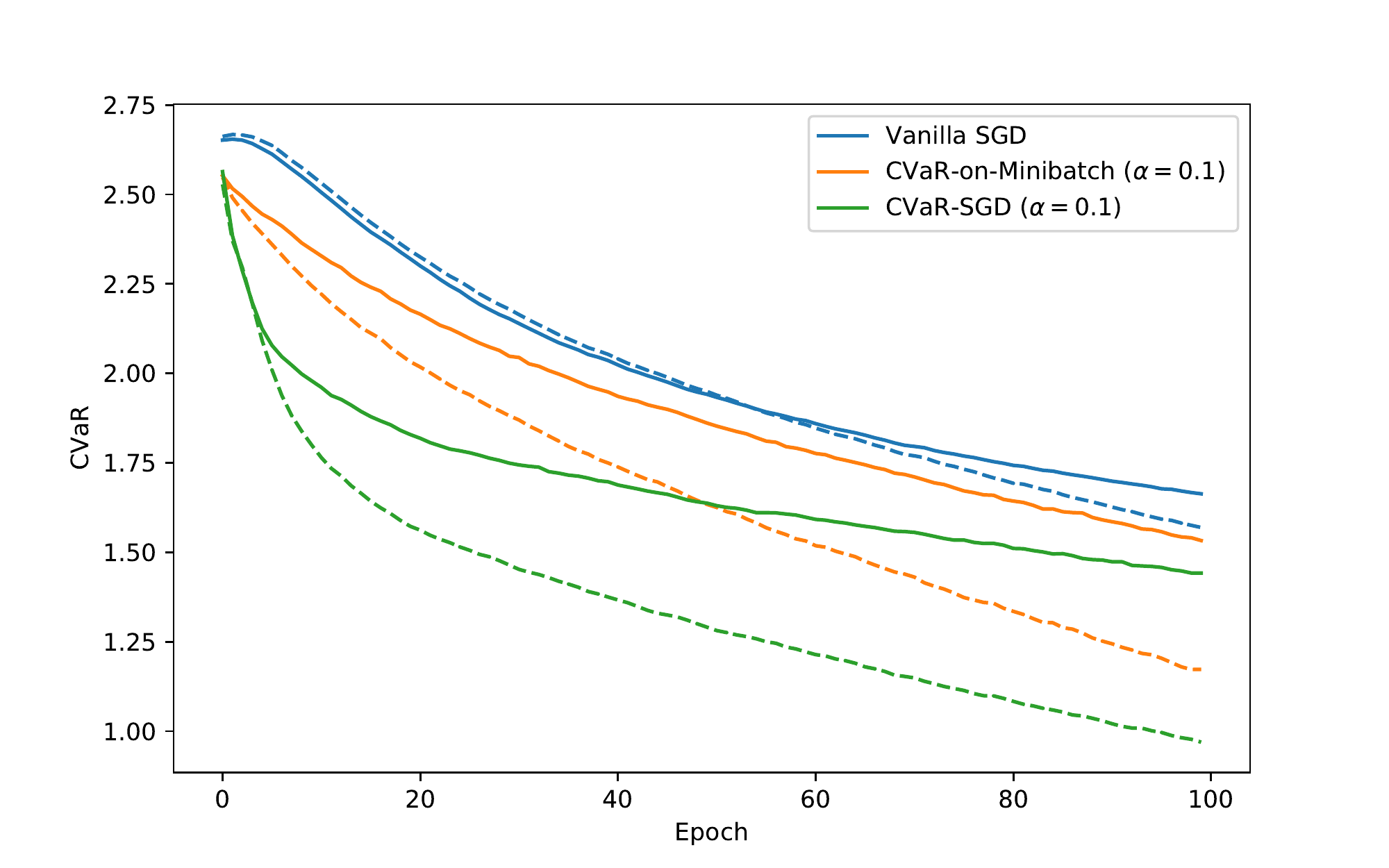}}
  \subfigure[Accuracy]{\includegraphics[width=.32\hsize]{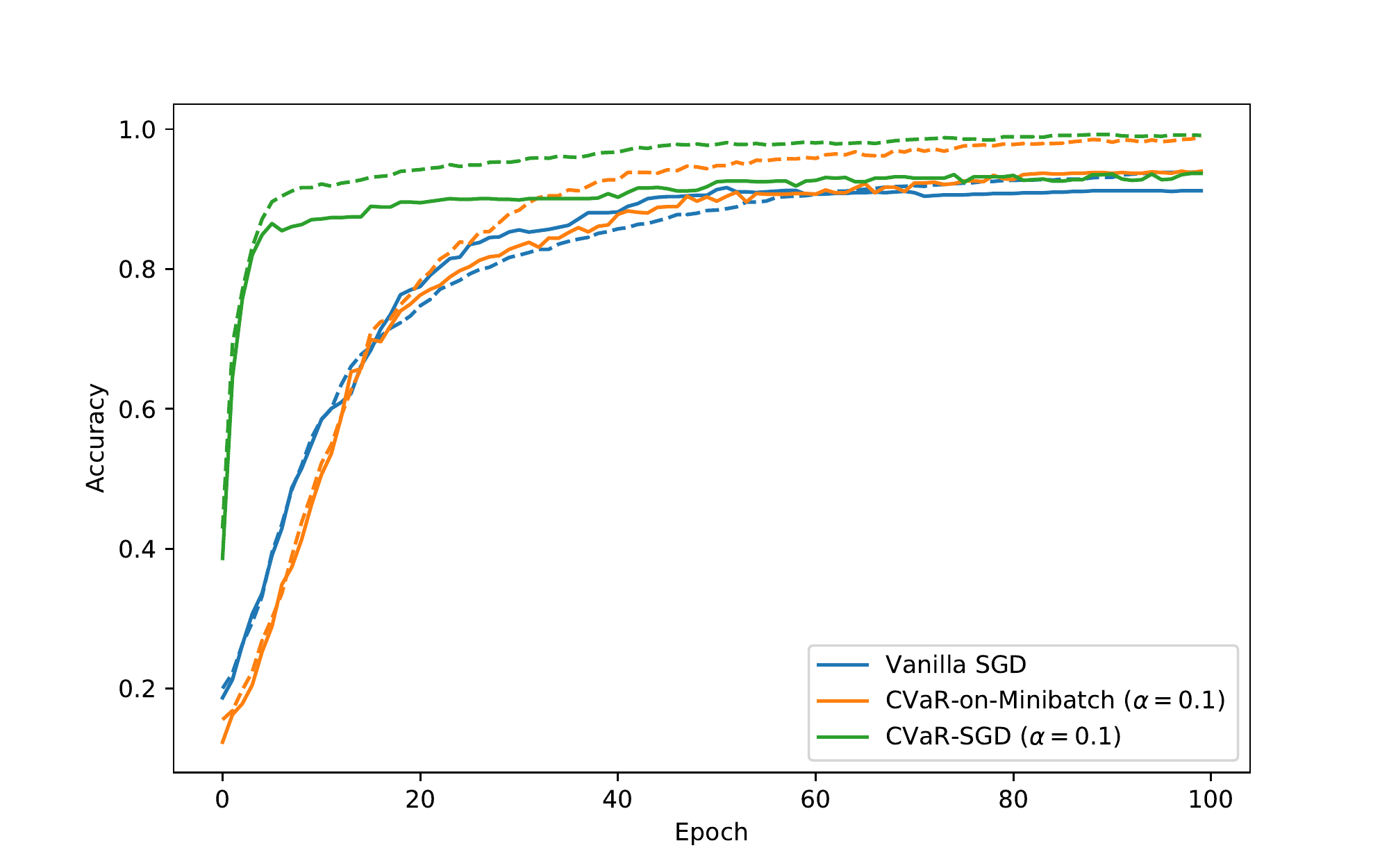}}
  \subfigure[Loss]{\includegraphics[width=.32\hsize]{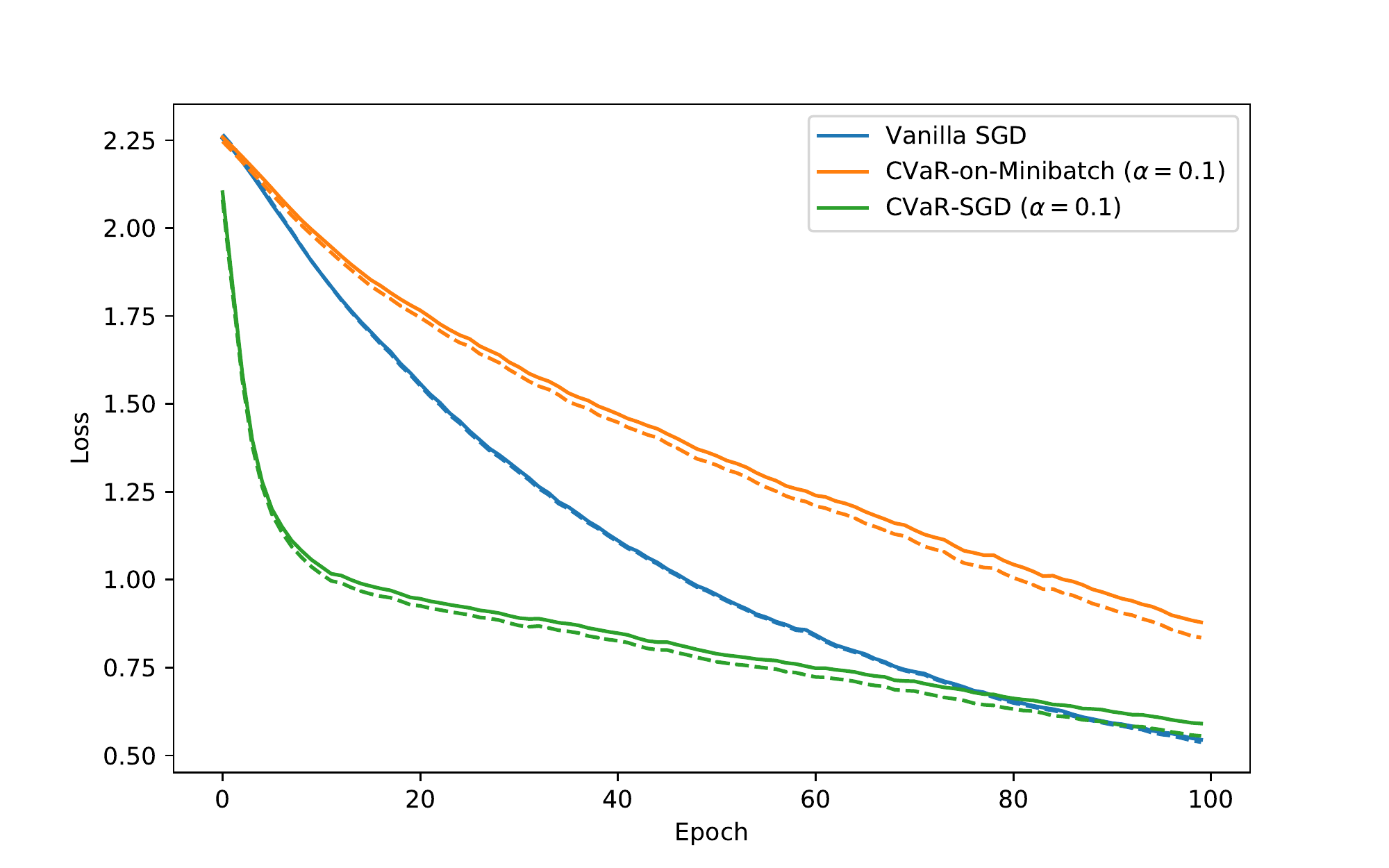}}
  \caption{Transition of the CVaR, accuracy, and average loss of the three-layer neural network on the digits dataset. Solid and dashed lines represent the results on the validation and training data, respectively.}\label{fig:digits-3-layer-curve}
\end{figure*}

\begin{figure*}[t!]
  \centering
  \subfigure[$\alpha=0.05$]{\includegraphics[width=.45\hsize]{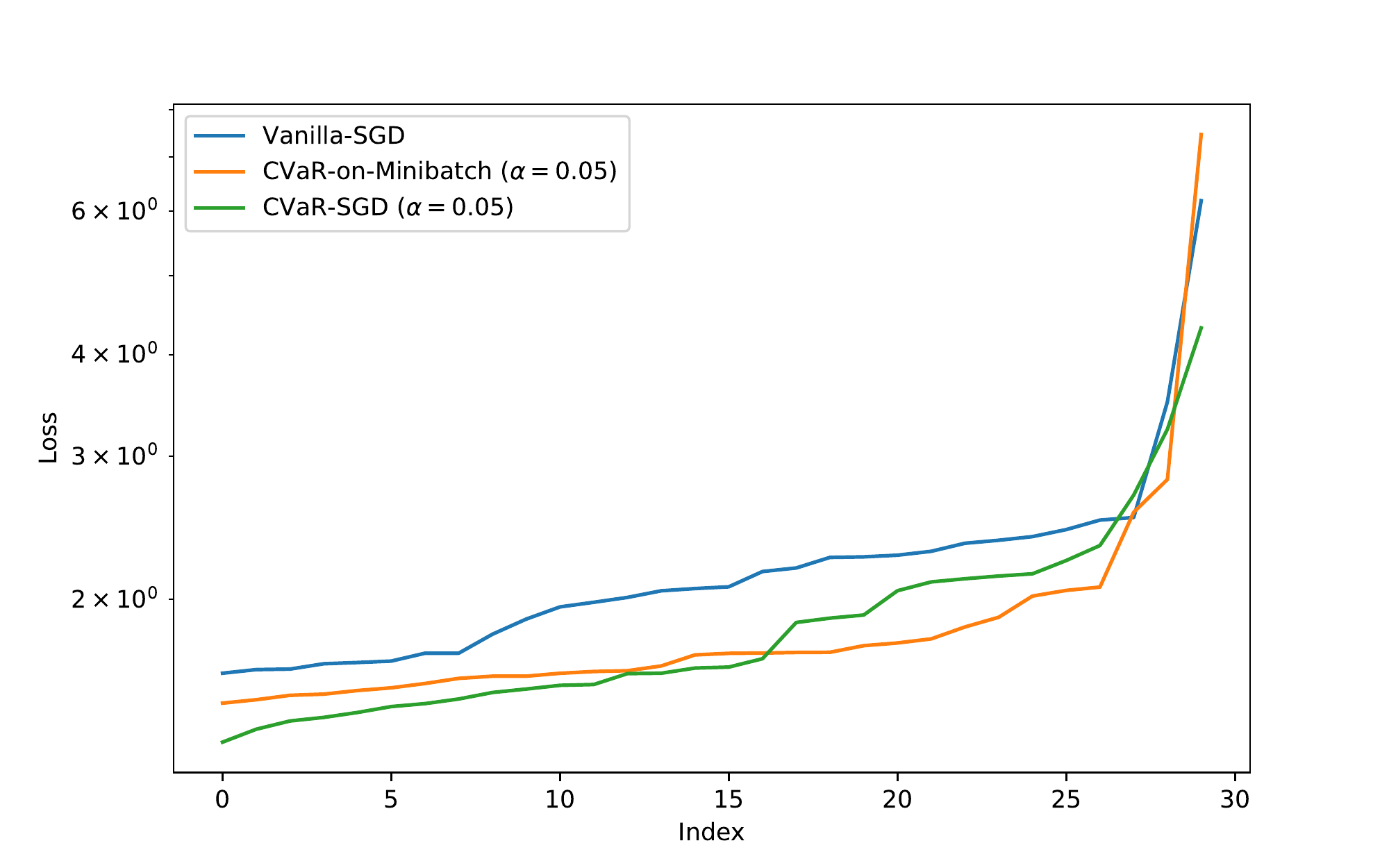}}
  \subfigure[$\alpha=0.1$]{\includegraphics[width=.45\hsize]{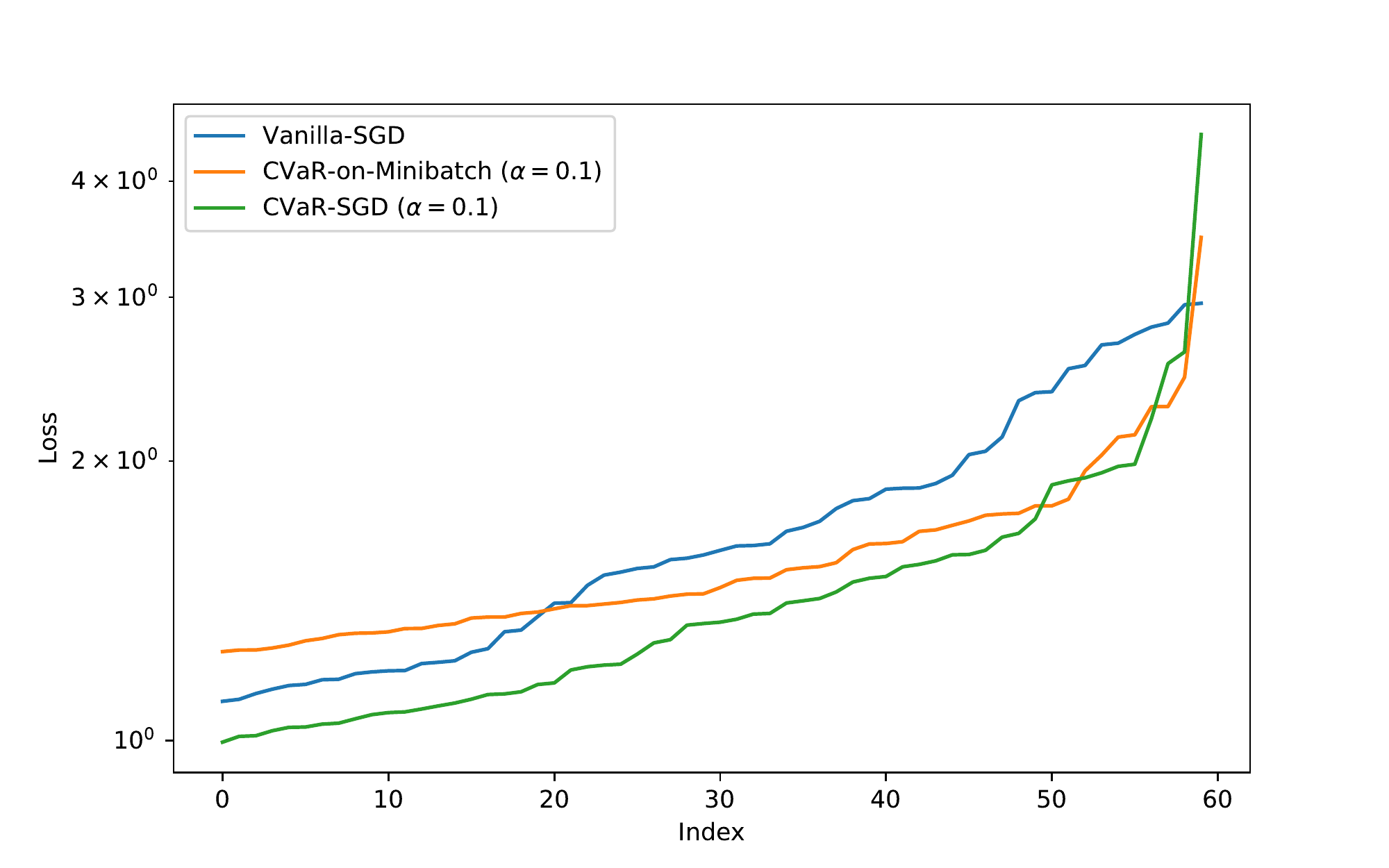}}
  \caption{Top $\alpha$-fraction of loss values of the three-layer neural network on the digits dataset sorted in increasing order}\label{fig:digits-3-layer-histogram}
\end{figure*}

\section{Further Experimental Results}\label{sec:further-experimental-results}
To demonstrate the effectiveness of our algorithms in the nonconvex setting, we conducted the same experiment as in Section~\ref{sec:experiments} using a three-layer fully connected neural network with ReLU activations having 100 hidden units in the middle layer.

See Figures~\ref{fig:summary-3-layer},~\ref{fig:digits-3-layer-curve}, and~\ref{fig:digits-3-layer-histogram} for the counterparts of Figures~\ref{fig:summary},~\ref{fig:digits-linear-curve}, and~\ref{fig:digits-linear-histogram}.

\end{document}